\documentclass[11pt]{article}
\usepackage[utf8]{inputenc} % allow utf-8 input
\usepackage[T1]{fontenc}    % use 8-bit T1 fonts
\usepackage{url}            % simple URL typesetting
\usepackage{booktabs}       % professional-quality tables
\usepackage{amsfonts}       % blackboard math symbols
\usepackage{nicefrac}       % compact symbols for 1/2, etc.
\usepackage{microtype}      % microtypography
\usepackage[sort&compress,square,numbers]{natbib}

\usepackage{latexsym}
\usepackage{amsmath}
\usepackage{amssymb}
\usepackage{amsthm}
\usepackage{epsfig}
\usepackage{xcolor}
\usepackage{graphicx}
\usepackage{bm}
\usepackage{enumitem}
\usepackage[toc,page]{appendix}
\usepackage{mathtools}
\usepackage{todonotes}
\usepackage{multirow}
\usepackage[mathscr]{euscript}
\usepackage{caption} 
\usepackage{mathtools}
\usepackage{todonotes}
\usepackage{subcaption}

\usepackage{xr}
\usepackage[top=1in, bottom=1.25in, left=1in, right=1in]{geometry}

\usepackage[numbers]{natbib}

\usepackage{amsmath,amsfonts,bm}
\usepackage{wrapfig}
\usepackage{sidecap}
\usepackage{thm-restate}

\usepackage{xcolor}         % colors
\usepackage{mathrsfs}
\usepackage{enumitem}
\definecolor{darkred_f}{RGB}{182, 85, 85}
\definecolor{darkblue_f}{RGB}{86, 116, 172}
\definecolor{darkorange_f}{RGB}{209, 136, 92}
\definecolor{darkgreen_f}{RGB}{106, 165, 110}

\definecolor{plot_blue}{RGB}{66, 153, 225}
\definecolor{plot_orange}{RGB}{237, 137, 54} 
\definecolor{plot_red}{RGB}{245, 101, 101}
\definecolor{plot_green}{RGB}{72, 187, 120}
\definecolor{plot_purple}{RGB}{159, 122, 234}
\definecolor{plot_green2}{RGB}{56, 178, 172}
\definecolor{plot_pink}{RGB}{237, 100, 166}
\definecolor{goldenyellow}{rgb}{1.0, 0.87, 0.0}

\definecolor{Gray}{gray}{0.9}
\definecolor{midgreen}{rgb}{0.1,0.5,0.1}
\definecolor{darkgray}{gray}{0.25}
\definecolor{lightblue}{rgb}{0.25,0.25,0.8}
\definecolor{mydarkblue}{rgb}{0,0.08,0.45}
\usepackage[colorlinks,
            linkcolor=mydarkblue,
            anchorcolor=blue,
            urlcolor=mydarkblue,
            citecolor=mydarkblue]{hyperref}
\def\eqref#1{equation~\ref{#1}}
\def\1{\bm{1}}
\def\0{\bm{0}}

\def\vb{{\bm{b}}}

\def\vd{{\bm{d}}}
\def\ve{{\bm{e}}}
\def\vf{{\bm{f}}}
\def\vg{{\bm{g}}}
\def\vh{{\bm{h}}}
\def\vi{{\bm{i}}}
\def\vj{{\bm{j}}}
\def\vk{{\bm{k}}}
\def\vl{{\bm{l}}}

\def\vu{{\bm{u}}}

\def\vx{{\bm{x}}}

\def\mA{{\bm{A}}}
\def\mB{{\bm{B}}}
\def\mC{{\bm{C}}}
\def\mD{{\bm{D}}}

\def\mH{{\bm{H}}}
\def\mI{{\bm{I}}}

\def\mK{{\bm{K}}}

\def\mM{{\bm{M}}}

\def\mQ{{\bm{Q}}}
\def\mR{{\bm{R}}}

\def\mT{{\bm{T}}}

\def\mW{{\bm{W}}}
\def\mX{{\bm{X}}}
\def\mY{{\bm{Y}}}
\def\mZ{{\bm{Z}}}

\def\mLambda{{\bm{\Lambda}}}

\DeclareMathAlphabet{\mathsfit}{\encodingdefault}{\sfdefault}{m}{sl}
\SetMathAlphabet{\mathsfit}{bold}{\encodingdefault}{\sfdefault}{bx}{n}

\def\gN{{\mathcal{N}}}

\def\gX{{\mathcal{X}}}

\def\sB{{\mathbb{B}}}

\def\sG{{\mathbb{G}}}

\def\sN{{\mathbb{N}}}

\def\sR{{\mathbb{R}}}
\def\sS{{\mathbb{S}}}

\def\sV{{\mathbb{V}}}

\newcommand{\E}{\mathbb{E}}

\DeclareMathOperator{\Tr}{Tr}

\newtheorem{theorem}{Theorem}%[section]
\newtheorem{lemma}{Lemma}%[section]
\newtheorem{claim}{Claim}%[theorem]
\newtheorem{corollary}{Corollary}%[section]
\newtheorem{example}{Example}%[section]
\newcommand{\op}{\mathrm{op}}

\def\vF{{\bm{F}}}

\newcommand{\spatial}{{\mathscr S}}
\newcommand{\frequency}{{\mathscr F}}

\usepackage{xcolor}         % colors

\definecolor{darkred_f}{RGB}{182, 85, 85}
\definecolor{darkblue_f}{RGB}{86, 116, 172}
\definecolor{darkorange_f}{RGB}{209, 136, 92}
\definecolor{darkgreen_f}{RGB}{106, 165, 110}

\definecolor{plot_blue}{RGB}{66, 153, 225}
\definecolor{plot_orange}{RGB}{237, 137, 54} 
\definecolor{plot_red}{RGB}{245, 101, 101}
\definecolor{plot_green}{RGB}{72, 187, 120}
\definecolor{plot_purple}{RGB}{159, 122, 234}
\definecolor{plot_green2}{RGB}{56, 178, 172}
\definecolor{plot_pink}{RGB}{237, 100, 166}
\definecolor{goldenyellow}{rgb}{1.0, 0.87, 0.0}

\newcommand{\upd}{{(d)}}

\def\myeqref#1{Eq.~(\ref{#1})}

\newcommand*\samethanks[1][\value{footnote}]{\footnotemark[#1]}

\title{Precise Learning Curves and Higher-Order Scaling Limits for Dot Product Kernel Regression}

\author{
    Lechao Xiao\thanks{LX, HH and TM contributed equally. HH's work was done while at Harvard University.}
    \\
  Google DeepMind\\
  \texttt{xlc@google.com} 
  \and
  Hong Hu\samethanks[1]
  \\
  University of Pennsylvania\\
  \texttt{huhong@wharton.upenn.edu}
  \and 
  Theodor Misiakiewicz\samethanks[1] \\ 
  Stanford University
  \\
  \texttt{misiakie@stanford.edu} 
  \and 
  Yue M. Lu
  \\
  Harvard University
  \\
  \texttt{yuelu@seas.harvard.edu}
  \and
  Jeffrey Pennington \\
  Google DeepMind\\
  \texttt{jpennin@google.com} 
}

\begin{document}

\maketitle

\begin{abstract}
As modern machine learning models continue to advance the computational frontier, it has become increasingly important to develop precise estimates for expected performance improvements under different model and data scaling regimes. Currently, theoretical understanding of the learning curves that characterize how the prediction error depends on the number of samples is restricted to either large-sample asymptotics ($m\to\infty$) or, for certain simple data distributions, to the high-dimensional asymptotics in which the number of samples scales linearly with the dimension ($m\propto d$). There is a wide gulf between these two regimes, including all higher-order scaling relations $m\propto d^r$, which are the subject of the present paper. We focus on the problem of kernel ridge regression for dot-product kernels and present precise formulas for the mean of the test error, bias, and variance, for data drawn uniformly from the sphere with isotropic random labels in the $r$th-order asymptotic scaling regime $m\to\infty$ with $m/d^r$ held constant. We observe a peak in the learning curve whenever $m \approx d^r/r!$ for any integer $r$, leading to multiple sample-wise descent and nontrivial behavior at multiple scales. We include a \href{https://colab.research.google.com/github/google/neural-tangents/blob/main/notebooks/learning_curves.ipynb
}{colab}\footnote{Available at: \href{https://tinyurl.com/2nzym7ym}{https://tinyurl.com/2nzym7ym}} notebook that reproduces the essential results of the paper. 
\end{abstract}

\section{Introduction}
Modern machine learning has entered an era in which scaling is arguably the most critical ingredient to improve performance. Recent breakthroughs such as GPT-3~\citep{kaplan2020scaling} and PaLM~\citep{chowdhery2022palm} have demonstrated that performance of various learning algorithms improves in a {\it predictable} manner as the amount of data and computational resources used in training increases. The functional relationships between performance and resources are loosely referred to as learning curves. While extrapolation of empirical learning curves is widely used to make predictions about how a model might perform when extra resources become available, a rigorous theoretical understanding is lacking. A fundamental obstacle in developing a detailed theoretical model of such learning curves is that they depend on many moving parts, e.g. the data distribution, the network architecture, the training algorithm, among others. In addition, even in the simplest possible settings, the learning curves can exhibit non-trivial structure that naive scaling laws fail to model, e.g. the well-known double-descent phenomenon~\citep{belkin2019reconciling, advani2020high}. 

In the past couple years, a large amount of effort from the community has improved our theoretical understanding of such phenomena and in some cases precise characterizations of learning curves have been obtained (see e.g., ~\citep{hastie2022surprises, adlam2020neural, mei2022generalization,montanari2019generalization,  loureiro2021learning}). These results have helped clarify several puzzling empirical observations, such as the origin of the double-descent peak~\citep{adlam2020understanding,lin2021causes,d2020double} and linear trends between in- and out-of-distribution generalization performance~\citep{tripuraneni2021covariate,tripuraneni2021overparameterization,mania2020classifier}, among many others. However, the precise predictions from many of these analyses have been possible only in the linear high-dimensional scaling regime in which the number of training samples $m$ scales linearly with the dimension $d$, i.e. $m\propto d$. In these asymptotics, the model's effective capacity is limited to linear functions of the features. In contrast, many state-of-the-art models operate in a regime where the amount of data is much larger than the data dimensionality; for example, large text corpora can contain trillions of tokens, whereas the effective input dimensionality of language models is at most millions. Therefore, going beyond the linear scaling regime ($m\propto d$) to higher-order scaling regimes ($m\propto d^r$) is essential in improving our understanding of modern machine learning systems, and is the focus of the current paper.

Several works have investigated the behavior of the learning curves for nonlinear scalings in the dot-product kernel or random features setting, but they have done so only in the noncritical regime where $m\not\propto d^{r}$~\citep{ghorbani2021linearized, mei2021generalization, liang2020multiple}.
\citep{bordelon2020spectrum, canatar2021spectral} also derive the closed-form predictions of the learning curves for both the critical and the noncritical scalings, but they have done so via nonrigorous statistical physics methods and a “Gaussian equivalence conjecture” \citep{dietrich1999statistical, hu2020universality,
gerace2020generalisation,goldt2020modeling, goldt2020gaussian,  goldt2022gaussian}. 
Rigorously extending these results to include the critical regime $m \propto d^{r}$ is nontrivial, both from the technical perspective, namely, proving a “Gaussian equivalence conjecture",  and also from the phenomenological perspective, as we shall see the critical behavior induces nonmonotonicity and multiple sample-wise descents.

\begin{figure}[t]
\centering
    \includegraphics[width=0.71\textwidth]{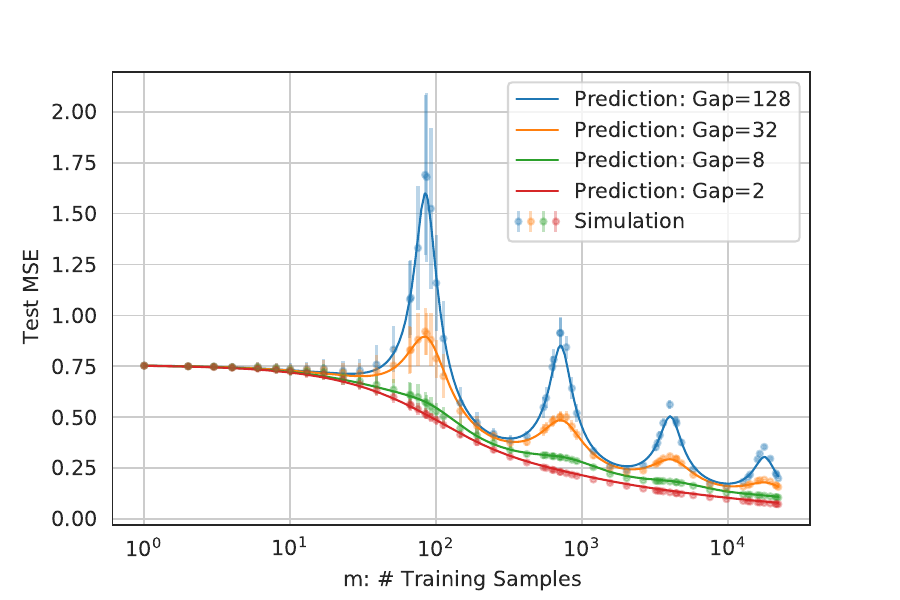}

\caption{{\bf Precise Sample-wise Learning Curves for One-hidden Layer CNN kernels.} 
    The theoretical predictions (\myeqref{eq:generalization formula}, solid lines) agree with finite-size simulations (markers) across several orders of magnitude and captures cases in which the curves are relatively simple ({\bf\color{plot_red}monotonically decreasing}, small spectral gap) and complex ({\bf\color{plot_blue} multiple-descent}, large spectral gap). Simulations are obtained from kernel regression with one-layer CNN kernels averaged over 50 runs. The input is of shape $d=d_0\times p$ with size $d_0=14$ and number of patches $p=6$. We vary the kernels by varying the ratio (aka spectral gap) between consecutive eigenspaces, where the ratio $\mathrm{Gap}\in [{\bf\color{plot_red}2}, {\bf\color{plot_green}8}, {\bf\color{plot_orange}32}, {\bf\color{plot_blue}128}]$.}
    \label{fig:multiple}
\vspace{-0.3cm}
\end{figure}

In this work, we obtain precise formulas for the sample-wise learning curves in the kernel ridge regression setting for a family of dot-product kernels for spherical input data in the polynomial scaling regimes $m\propto d^r$ for all $r\in\sN^*$. This family of kernels includes the neural network Gaussian Process (NNGP) kernels and Neural Tangent Kernels (NTK) associated with multi-layer fully-connected networks or convolutional networks. Both kernels serve as important starting points towards a deeper understanding of neural networks as they often capture the first order learning dynamics of neural networks in certain scaling limits \citep{jacot2018neural, lee2019wide, arora2019exact}.

\subsection{Contributions}
Our primary contributions are to establish the following, for data drawn uniformly from the sphere:
\begin{enumerate}
    \item The empirical spectral density of the Gram matrix induced by degree-$r$ spherical harmonics converges to a Marchenko-Pastur distribution when $(d^r/r!)/m$ converges to a positive constant as $d\to\infty$ (Theorem~\ref{Theorem:mp});
    \item A precise closed-form formula for the sample-wise learning curves for dot-product kernel regression when $m\propto d^r$ for all $r\in\sN^*$ as $d\to\infty$ (Theorem~\ref{thm:generalization-fc-kernel});
    \item Empirically, the theoretical predictions agree with finite-size simulations surprisingly well even in the strong finite-size correction regime (Fig.~\ref{fig:multiple});
    \item An extension of the above results to convolutional kernels (Section~\ref{sec:cnn}).
\end{enumerate}
Finally, we note that our results also assume the high-degree coefficients of the label function to be random and isotropic; see \myeqref{eq:isotropic-random}. It remains an open question to prove similar results\footnote{See Sec.~\ref{sec:experimental-setup} for empirical evidences in favor of these results. } when the label function is deterministic.

\section{Notation and Setup} 
Let $\bm\gX =\sS_{d-1}$ denote the input space, where $\sS_{d-1}$ is the unit sphere in $\mathbb R^d$ and $\bm\gX$ is equipped with the normalized uniform measure $\sigma$. We use $\Delta_d$ to represent any quantity (a scalar, vector or a matrix) with $|||\Delta_d|||\to 0$ as $d\to\infty$ (in probability if $\Delta_d$ is stochastic), where $|||\cdot |||$ can be the absolute value of a scalar, the norm of a vector or the operator norm of a matrix.

Let $\mX\in \sR^{m\times d}$ be the training inputs where the $i$-th row of $\mX$ is $\vx_i^\top$. We assume $\{\vx_i\}_{i\in [m]}$ is sampled uniformly, iid from $\bm\gX$. The label function $f: \sS_{d-1}\to \sR$ will be defined in Section~\ref{sec:generalization}.
Let $K= K^\upd$ be a dot-product kernel defined on $\sS_{d-1}\times \sS_{d-1}$, i.e., $K(\vx, \vx') = h(\vx^\top\vx') $ for some function $h\in [-1,1 ]\to \mathbb R$. We assume $h$ has the following decomposition 
\begin{align}
    h(t) = \sum_{k=1}^\infty \hat h_k^2  P_k(t), \quad \mathrm {with}\quad \sum_{k=1}^\infty \hat  h_k^2 <\infty\,, 
\end{align}
where $P_k$ is the $k$-th order Legendre polynomials in $d$ dimensions. 
For simplicity, we assume $\hat \vh = (\hat h_k)_{k\geq 1}$ is a sequence that is independent of $d$ and $\hat h_k\neq 0$ for all $k\leq k_0$ where $k_0$ is sufficiently large. As such, we can decompose the kernel function using sperical harmonics,  
\begin{align}
    K(\vx, \vx') = \sum_{k=1}^\infty \sigma_k^2\sum_{l \in [N(d, k)]}Y_{k, l}(\vx) Y_{k, l}(\vx')
    = \sum_{k=1}^\infty \sigma_k^2 Y_k(\vx)^\top Y_k(\vx')\,,
\end{align}
where $Y_{k, l}$ is the $l$-th spherical harmonic of degree $k$, $N(d, k) = d^k/k! + O(d^{k-1})$ is the total number of degree $k$ spherical harmonics in $d$ dimensions, 
$\sigma_k^2 = \hat h_k^2/N(d, k)$ is the eigenvalue of $Y_{kl}$, 
and $Y_k(\vx)$ is the column vector $[Y_{k, l}(\vx)]_{l\in [N(d, k)]}^\top$. We also denote by $Y_k(\mX)$ the $m\times N(d, k)$ matrix whose $i$-th row is $Y_k(\vx_i)^\top$.

\section{Structure of the Empirical Kernel and Marchenko-Pastur Distribution}
The structure of the empirical kernel matrix $K(\mX, \mX)$ plays a critical role in characterizing the sample-wise test error for the kernel ridge regressor associated to $K$. We assume the training set size scales polynomially, i.e.  $m\sim d^r$ for some positive integer $ r\in\sN^*$. Decompose this kernel into low-, critical- and high-frequency modes as follows,
\begin{align}
    K(\mX, \mX) = \sum_{k <r} \sigma_k^2Y_k(\mX)Y_k(\mX)^\top + 
    \sigma_r^2Y_r(\mX)Y_r(\mX)^\top + \sum_{k>r}\sigma_k^2 Y_k(\mX)Y_k(\mX)^\top\,.
\end{align}
The low- and high-frequency parts have simple structures since $N(d,k)/m$ either diverges to infinity or converges to zero with rate as least $d^{\pm 1}$, yielding concentration that results in significant simplification. To be precise, 
for high-frequency modes $k>r$, $Y_k(\mX)$ is a ``fat" matrix and $Y_k(\mX)Y_k(\mX)^\top/N(d, k) = \mI_{m} + \Delta_d$ where $\Delta_d$ vanishes as $d\to\infty$ \citep{mei2021generalization}. Thus, the high-frequency parts behave like a regularizer in the following sense,
\begin{align}
    \sum_{k>r}\sigma_k^2 Y_k(\mX)Y_k(\mX)^\top = 
    \sum_{k>r}\sigma_k^2 N(d, k) \mI_{m} + \Delta_d
    =
    \left(\sum_{k>r}\hat h_k^2 \right) \mI_{m} + \Delta_d\,.
\end{align}
On the other hand, when $k<r$, $Y_k(\mX)$ is a $m \times N(d, k)$ ``tall" matrix with $N(d, k)/m =O(d^k/m) = O(d^{-(r-k)})\to 0 $. Similarly, \citet{mei2021generalization} show that $Y_k(\mX)^\top Y_k(\mX)/m= \mI_{N(d,k)} +\Delta_d$, implying that 
when restricted to the subspace spanned by low-frequency functions $\{Y_{kl}\}_{k<r}$, the regressor associated to the empirical kernel $K(\mX, \mX)$ acts like a pure multiplicative scaling.

It remains to understand the critical-frequency mode $Y_r(\mX)^\top Y_r(\mX)$. It turns out that if $N(d,k)/m\to \alpha \in (0, \infty)$,  
then the empirical spectral measure of the random matrix $Y_r(\mX)^\top Y_r(\mX)/m$ converges to the Marchenko-Pastur distribution $\mu_\alpha$, whose density is given by 
\begin{align}
    \mu_{\alpha}(t) = \left(1-\frac1\alpha\right)^{+} \delta_{0}(t) +  \frac{\sqrt{(\alpha_+ - t)(t-\alpha_-)}}{2\pi \alpha t}\1_{[\alpha_-, \alpha_+]}(t), \,
    \mathrm{where} \,\,
    \alpha_{\pm}=(1\pm \sqrt \alpha)^2. 
\end{align}
where $\delta_0(t)=0$ if $t\neq 0$ else $1$. See Fig.~\ref{fig:mp-distribution} for visualizations of $\mu_\alpha$. The $r=1$ case is obvious as $Y_1(\mX) = c_d \mX$ for some normalizing constant $c_d$ and it is clear $\frac 1m Y_1(\mX)^\top Y_1(\mX) = \frac {c_d^2} m \mX^\top\mX$ converges to the Marchenko-Pastur distribution $\mu_{\alpha}$ if $d/m\to \alpha \in (0, \infty)$ as $d\to\infty$ \citep{tao2012topics}. Our first result show that this result continues to hold for all degrees. 
\begin{figure}
    \centering
    \includegraphics[width=0.8\textwidth]{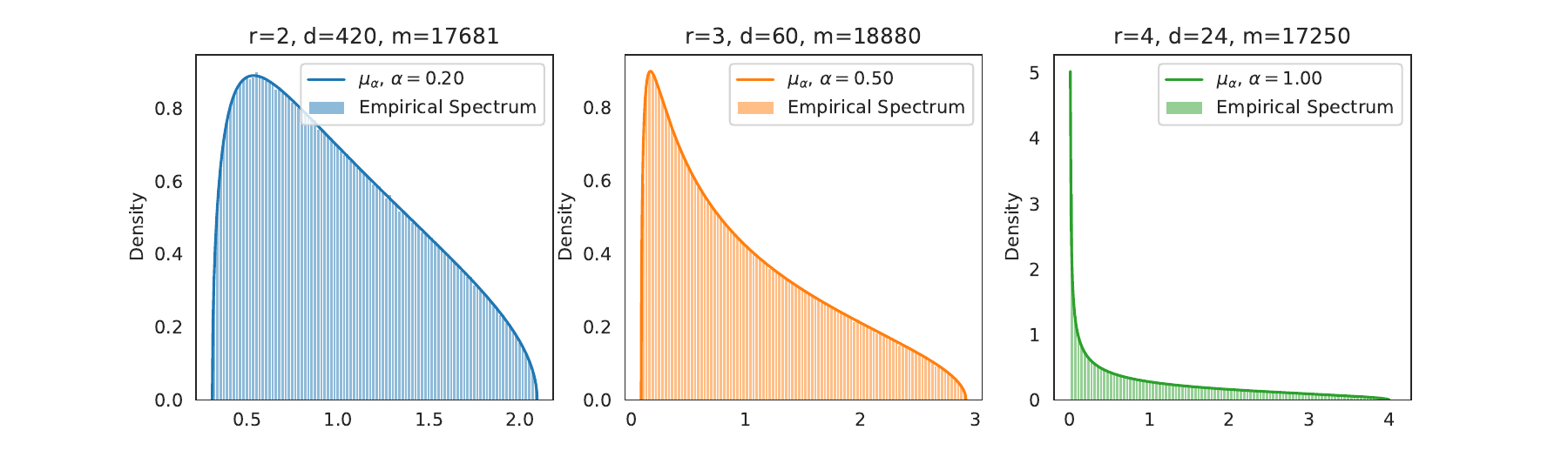}
    \\
    \includegraphics[width=0.8\textwidth]{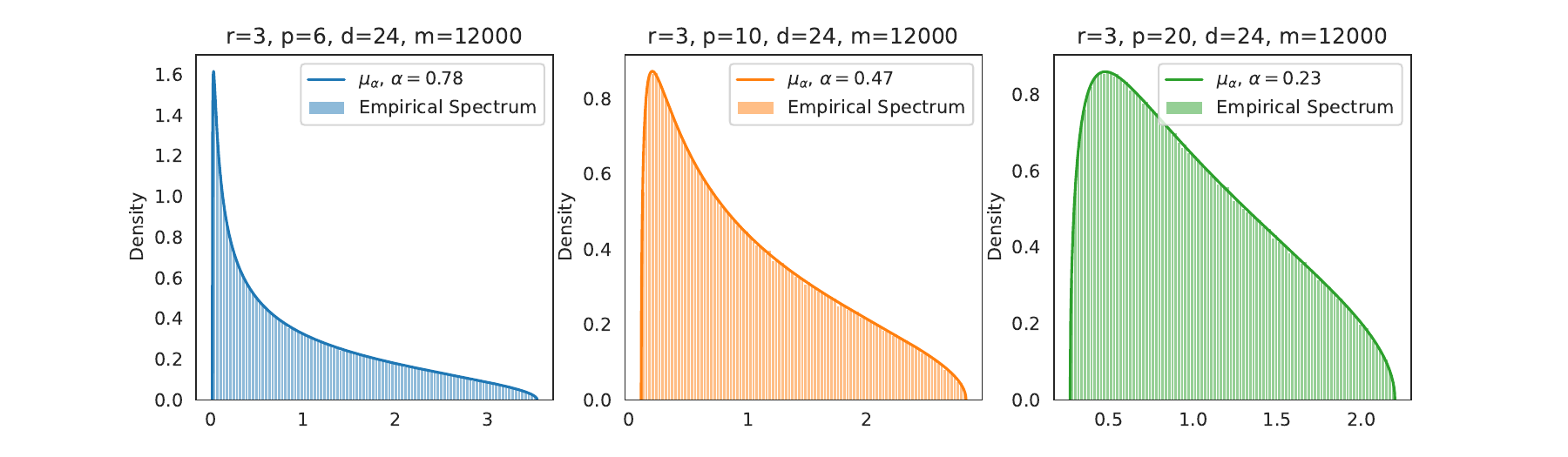}
    \caption{{\bf Marchenko-Pastur Distribution of Spherical Harmonics.} Top: the empirical distribution of product kernels $Y_r(\mX)Y_r(\mX)^\top/N(d,r)$ vs theory prediction from $\mu_\alpha$ for various degrees $r$, input dimensions $d$ and number of samples $m$ as indicated in the titles. Bottom: the empirical distribution of the CNN kernel $Y_r(\mX)Y_r(\mX)^\top/pN(d,r)$ vs theoretical prediction. We fix $r, d$ and $m$ but varying the number of patches $p \in \{6, 10, 20\}$.}
    \label{fig:mp-distribution}

\end{figure}

\begin{theorem}\label{Theorem:mp}
For fixed $r\in \sN$ and $\alpha\in (0,\infty)$, if $N(d, r)/m\to \alpha\in (0,\infty)$ as $d\to\infty$, then the empirical spectral distribution of $\frac 1m Y_r(\mX)^\top Y_r(\mX)$ converges in distribution to the Marchenko-Pastur distribution $\mu_{\mathrm MP(\alpha)}$ . 
\end{theorem}
In the top panel of Fig.~\ref{fig:mp-distribution}, we generate the empirical spectra\footnote{In the plot, we generate the spectra of the kernel matrix $Y_r(\mX)Y_r(\mX)^\top$ instead of the covariate matrix $Y_r(\mX)^\top Y_r(\mX)$. Although both of them have the same set of non-zero eigenvalues, 
the former can be easily implemented via Legendre polynomials $P_r(\vx^\top\vx')$.} of $\frac 1m Y_r(\mX)^\top Y_r(\mX)$ for various values of $r$,  $d$, and $\alpha$. The Marchenko-Pastur distribution $\mu_\alpha$ perfectly captures the empirical measures of the random matrices $\frac 1m Y_r(\mX)^\top Y_r(\mX)$ for all $r$ considered. We sketch the main steps of the proof of the theorem below; 
see Sec.~\ref{sec:proof-pm} for the whole proof.
\begin{proof}[Sketch of Proof]
From~\citet[Theorem 1.1]{bai2008large}, it suffices to prove concentration of the following quadratic forms: for every sequence of $N(d,k)\times N(d, k)$ matrices $\{\mA_{d}\}$ with operator norm $\|\mA_d\|_{\mathrm{op}}\leq 1$, the variance 
\begin{align}
    N(d,r)^{-2}\sV( Y_r(\vx)^\top \mA_d  Y_r(\vx) - \mathrm{Tr}(\mA_d))\to 0 \quad \mathrm{as}\quad d\to\infty \, . 
\end{align}
For the purpose of illustration, we assume $\mA \equiv \mA_d$ is a diagonal matrix. Then we only need to show
\begin{align}\label{eq:mp-concentration}
    N(d,k)^{-2}\sum_{l, l'\in [N(d, k)]} A_{ll}A_{l'l'}(\E_\vx Y_{k,l}^2(\vx)Y_{k,l'}^2(\vx) -1)\to 0  .
\end{align}
By hypercontractivity of spherical harmonics \citep{beckner1992sobolev}, 
\begin{align}
    \E_\vx Y_{k,l}^2(\vx)Y_{k,l'}^2(\vx) \leq  \left(\E_\vx Y_{k,l}(\vx)^4  \E_\vx Y_{k,l'}^4(\vx)\right)^{1/2} \leq C_k\E_\vx Y_{k,l}(\vx)^2
    \E_\vx Y_{k,l'}(\vx)^2 = C_k,
\end{align}
where $C_k$ is some absolute constant. 
Since $|A_{ll}|\leq \|\mA\|_{\mathrm{op}}\leq 1$, we can drop any $o(N(d, l)^2)$ pairs of $(l, l')$ in \myeqref{eq:mp-concentration}. We show that for the remaining pairs $(l, l')$, the eigenfunctions are asymptotically uncorrelated in the sense 
\begin{align}
    \E_\vx Y_{k,l}^2(\vx)Y_{k,l'}^2(\vx) = \E_\vx Y_{k,l}^2(\vx) \E_\vx Y_{k,l'}^2(\vx) + O(d^{-1}) = 1 + O(d^{-1})
\end{align}
which implies \myeqref{eq:mp-concentration}. 
\end{proof}

\section{Generalization Error of Dot-Product Kernel Regression}
\label{sec:generalization}
In this section, we establish the {\it average} generalization error for the kernel regression in the asymptotic regime $N(d,r)/m\to \alpha$, for some $\alpha\in (0, \infty)$ and $r\geq 1$ fixed. We assume the label function $f\in L^2(\mathbb S_{d-1})$ is given by 
\begin{align}
    f(\vx) = \sum_{k\geq 1 } \sum_{l\in [N(d, k)]} \hat f_{k l} Y_{k l}(\vx) = 
     \sum_{k\geq 1} \hat\vf_k^\top Y_k(\vx)\,,
\end{align}
where $\hat f_{kl}$ are the ``Fourier" coefficients and $\hat \vf_k = [f_{k l}]_{l\in N(d, k)}^\top$. We need to make a technical assumption that for $k', k\geq r$
\begin{align}\label{eq:isotropic-random}
    \E \hat \vf_k = \0,  \, \quad \quad \E \hat \vf_k \hat \vf_k^\top =   \frac{\hat F_k^2}{N(d, k)}  \mI_{N(d, k)}
    \quad \mathrm{and}\quad \E \hat \vf_k \hat \vf_{k'}^\top = \0_{N(d, k) \times N(d, k')}
\end{align}
i.e. $\hat \vf_k$ is centered with isotropic covariance and $\{\hat \vf_k\}_{k\geq r}$ are mutually uncorrelated. 
Note that we allow $\hat f_{kl}$ to be deterministic for $k<r$. 
We let $\vF = (\hat F_k)_{\geq 1}$ be a fixed sequence with
$ \sum_{k\geq 1}\hat F_k^2<\infty$, where $\hat F_k^2 = \sum_{l\in [N(d, k)]} \hat f_{kl}^2$ for $k<r$. 
For convenience, set $\hat F_{>j}^2 = \sum_{k>j}\hat F_k^2$ (similarly for $\hat F^2_{\geq j}$,  $\hat F^2_{\leq j}$, etc.) and use $\vf$ to denote the random vector $\{\hat f_{kl}\}_{kl}$. 
Given training inputs $\mX$ and observed labels  $\mY = f(\mX)+{\bm\epsilon}$, where $\bm\epsilon \sim \mathcal N(\0, \sigma_\epsilon^2 \mI_m)$ is the noise,  the prediction using kernel function $K$ is given by 
\begin{align}
    y(\vx) = K(\vx, \mX) (K(\mX, \mX)+ \lambda \mI_m)^{-1}(f(\mX) + \bm\epsilon)\, . 
\end{align}
Here $\lambda\geq 0$ is the regularization. As such, the {\it mean} test error over the random labels is given by   
\begin{align}\label{eq:kernel regression error}
\mathrm{Err}(\mX;\lambda,  \vF,  \hat\vh) = \E_{\vf}\mathrm{Err}(\mX; \lambda, \vf, \hat\vh)
\quad \mathrm{where} 
\quad 
    \mathrm{Err}(\mX; \lambda, \vf, \hat\vh) = \E_{\vx, \bm \epsilon} |y(\vx) - f(\vx)|^2 \, .
\end{align}
To state our results, we need to introduce two functions $\chi_B$ and $\chi_V$ which are related to the bias and variance in the generalization error, 
\begin{align}
    &\chi_B(\alpha, \xi) = \int (1 + \xi  t)^{-2} \mu_{\alpha}(t)dt
    \quad \mathrm{and}\quad 
    % \\
    \chi_V(\alpha, \xi)  = \alpha\xi ^2
    \int {t}(1+ \xi t) ^{-2} \mu_{\alpha} (t) dt \,. 
\end{align}
Both $\chi_B$ and $\chi_V$ have closed-form representations; see Sec \ref{sec:computing integrals}. 
Define the effective regularization associated to the $r$-th order scaling to be 
\begin{align}\label{eq:effective-regularization}
    \xi_r(\hat\vh,\lambda, \alpha) = \frac{\hat h_r^2}{\alpha(\lambda + \hat h_{>r}^2)}
\end{align}
Finally, 
we define the bias and variance associated to the $r$-th order scaling to be
\begin{align}
\label{eq:bias}
 \mathrm{B}_r(\alpha) =     \mathrm{B}_r(\alpha; \lambda,  \vF, \hat\vh)
    =& \chi_B(\alpha, \xi_r(\hat\vh,\lambda, \alpha)) \hat F_r^2 + \hat F_{>r}^2
    \\
    \mathrm{V}_r(\alpha) = \mathrm{V}_r(\alpha; \lambda, \vF, \hat\vh)
    =&\chi_V\left(\alpha, \xi_r(\hat\vh,\lambda, \alpha)\right)\left(\hat F_{>r}^2 + \sigma_\epsilon^2\right)
    \label{eq:variance}
\end{align}
The following is our main result, which characterizes the test error in the asymptotic regime $m\propto d^r$. 
\begin{theorem}\label{thm:generalization-fc-kernel}
Let $\alpha\in (0, \infty)$ and $r\geq 1$ be fixed. Assume $N(d,r)/m\to \alpha$ as $d\to\infty$. Then the average test error is given by 
\begin{align}\label{eq:generalization formula}
    \mathrm{Err}(\mX; \lambda, \vF, \hat\vh)= 
   \mathrm{B}_r(\alpha; \lambda,  \vF, \hat\vh) + \mathrm{V}_r(\alpha; \lambda, \vF, \hat\vh) 
   + \Delta_d,   % \quad \mathrm{where}\quad \Delta_d\to 0\mathrm{in\, probability.}
\end{align}
where $\Delta_d\to 0$ in probability. 
\end{theorem}
\subsection{Interpretations}
We provide some high-level interpretations of the bias term $\mathrm{B_r}$ and variance term $\mathrm{V_r}$. 

\paragraph{The Bias.} From \myeqref{eq:bias}, 
the regressor learns all low-frequency modes ($k<r$) but none of the high-frequency modes ($k>r$) as the bias $\mathrm{B}_r$ contains no low-frequency modes (i.e. $k<r$) but all high-frequency modes $\hat F_{>r}^2$. Importantly, the regressor is progressively learning the critical-frequency mode $Y_r$ as the training size $m=\frac 1 \alpha N(d, r)$ increases, i.e. from $\alpha=\infty$ to $\alpha=0^+$ since $\chi_B(\alpha, \xi_r( \hat\vh,  \lambda, \alpha))\to 1$ if $\alpha\to \infty$ and $\chi_B(\alpha, \xi_r( \hat\vh,  \lambda, \alpha))\to 0$ if $\alpha\to 0^+$. See Fig.\ref{fig:BV-decompose} for the illustration. 

\begin{wrapfigure}{r}{0.45\textwidth}
    \centering
    \includegraphics[width=0.45\textwidth]{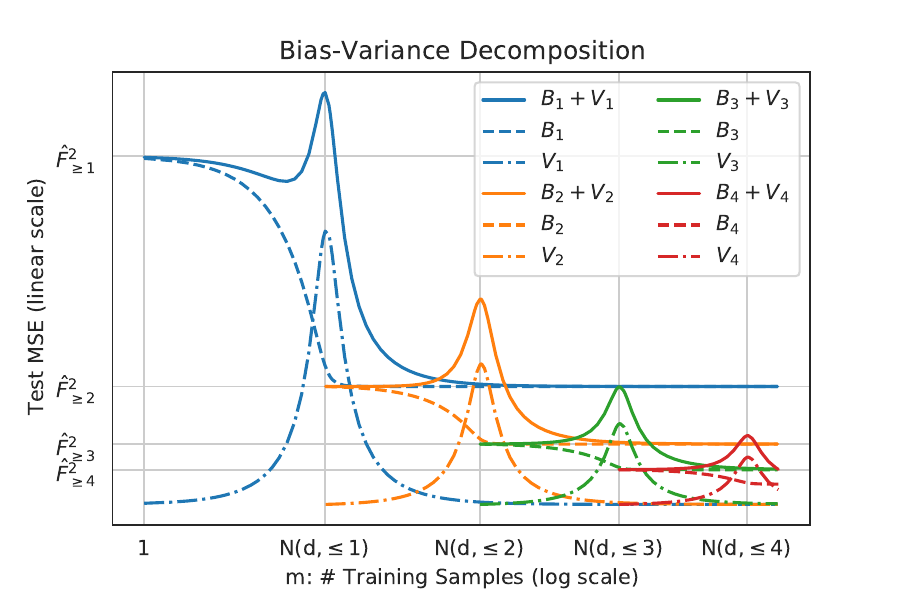}
    \caption{{\bf Multi-scale Bias-Variance Decomposition.}
    Theoretical predictions of the bias and variance from Eq.\ref{eq:bias} and Eq.\ref{eq:variance}. For each $r$, the variance is non-monotonic and has a peak at $N(d, \leq r) =\sum_{k\leq r}N(d,k) $.}
    \label{fig:BV-decompose}

\end{wrapfigure}

\paragraph{The Variance.} From \myeqref{eq:variance}, the variance term $\chi_{\mathrm{V}}$ treats all high-frequency modes $\hat F_{>r}^2$ the same as the noise term $\bm \epsilon$. Moreover, $\chi_{\mathrm{V}}\to 0$ as $\alpha\to 0$ or $\infty$ and is peaked at $\alpha=1$. The height of the peak depends on the effective regularization $\xi_r$ and it diverges to infinity with rate $\xi_r^{\frac 12}$ as $\xi_r\to\infty$. Indeed, when $\alpha=1$ and $\xi^{-1}/2\leq t\leq \xi^{-1}$, we have $t(1+ \xi t) ^{-2} \mu_{\alpha} (t) \propto \xi^{-1/2} $ which implies $\chi_V(1, \xi) \geq \xi^2 \int t(1+ \xi t) ^{-2} \mu_{\alpha} (t)  \1_{\xi^{-1}/2\leq t\leq \xi^{-1}}dt \propto \xi^{1/2}$.

Finally, \myeqref{eq:generalization formula} not only gives precise generalization formula (up to a vanishing term $\Delta_d$) when $m\approx N(d, r)\sim d^{r}$ but also when $d^{r-1+\delta} \lesssim m \lesssim d^{r-\delta}$ (i.e. when ``$\alpha= \infty$") and when $d^{r+\delta}\lesssim m \lesssim m^{r+1-\delta}$ (i.e. when ``$\alpha=0^+$") for any $\delta\in (0, 1/2)$.  Indeed, in the non-critical scaling regime $d^{r-1+\delta} \lesssim m \lesssim d^{r-\delta}$,  $\alpha=N(d, k)/m\to \infty$ as $d\to\infty$ and 
\begin{align}
    \mathrm{B_r}(\alpha=\infty) +\mathrm{V_r}(\alpha=\infty)= \hat F_{\geq r}^2 + 0 = \hat F_{\geq r}^2\,.
\end{align}
As such, the regressor learns all low-frequency modes but none of the critical- and high-frequency modes ($k\geq r$), which is consistent with the result in \citep{ghorbani2021linearized, mei2021generalization}. A similar argument also shows $\mathrm{B_r}(\alpha=0^+) +\mathrm{V_r}(\alpha=0^+) = \hat F_{> r}^2$, namely, the regressor also learns the $r$-frequency mode. This observation implies that we can glue together \myeqref{eq:generalization formula} for $r\geq 1$ and remove all duplicate terms to generate a sample-wise learning curve (LC):
\begin{align}
    \mathrm{LC}(m;  \lambda, \vf, \hat\vh) = 
    \sum_{r\geq 1} \left(
    \mathrm{B}_r\left(\frac{N(d, r)}{m}; \lambda,  \vf, \hat\vh\right) - (r-1) \hat  F_r^2\right) + \mathrm{V}_r\left(\frac{N(d, \leq r)}{m}; \lambda, \vf, \hat\vh\right) 
\end{align}
where $N(d, \leq r)=\sum_{k=1}^rN(d, k)$.
The ``$-(r-1) \hat  F_r^2$" term in the above equation is due to the fact that $\hat  F_r^2$ is over-counted $(r-1)$ many times (one in each $\mathrm{B}_k$ for $k=1, ..., (r-1)$.) It is worth mentioning that using $\alpha=N(d, \leq r)/{m}$ rather than
$\alpha={N(d, r)}/{m}$ in the variance $\mathrm{V}_r$ captures the finite-size correction more accurately. See \myeqref{eq:finite-width-effect-variance}.  
\begin{corollary}\label{cor:learning-curve}
If, for $1\leq r \in \sN$, (1) $N(d,r)/m \to\alpha$ for some $\alpha \in(0, \infty)$, or (2) $d^{r-1+\delta}\lesssim m \lesssim d^{r-\delta}$ for some $\delta\in(0, 1/2)$, then
\begin{align}\label{eq:learning-curves}
    \mathrm{Err}(\mX; \vf, \hat\vh, \lambda) = 
    \mathrm{LC}(m; \lambda, \vf, \hat\vh) + \Delta_d
\end{align}
where $\Delta_d \to 0$ in probability as $d\to\infty$. 
\end{corollary}
Recall that for each $r$, the variance term $\mathrm{V}_r$ could diverge to infinity as $\xi_r\to\infty$ at $\alpha=1$. Thus we might expect a peak in the learning curve for each $r$, yielding the multiple-descent phenomenon, as shown in Fig.\ref{fig:multiple}. However, such phenomena can disappear by making the heights of the peaks small via choosing $\xi_r$ small. We will discuss this point in the experimental section.

\subsection{Proof Sketch}
The proof of this theorem is quite involved; see Sec.\ref{sec:proof-of-generalization}.  For simplicity, we assume the observed labels are noiseless, i.e. $\sigma_\epsilon^2=0$. 
An ingredient is to understand the structure of the operator
\begin{align}
    \mT_{K}f(\vx) = K(\vx, \mX) (K(\mX, \mX)+ \lambda \mI_m)^{-1}f(\mX)\,.
\end{align}
The high-level strategy is as follows. We decompose the function into low-, critical- and high-frequency parts $f=f_{<r} + f_r + f_{>r}$. As such, the test error is roughly 
\begin{align*}
    \mathrm{Err}(\mX) \approx 
     \mathrm{Err}_{<r}(\mX)  + \mathrm{Err}_r(\mX)  +  \mathrm{Err}_{>r}(\mX) \,\, \mathrm{where}\,\,
     \mathrm{Err}_{r}(\mX) = 
 \E_\vf\E_\vx|\mT_{K}f_{r}(\vx) - f_{r}(\vx)|^2\,,
\end{align*}
and similarly for $\mathrm{Err}_{<r}(\mX)$ and $\mathrm{Err}_{>r}(\mX)$. The next step is to estimate each part separately. 

\paragraph{Low-frequencies.} Using the fact that the low-frequency parts of the kernel function $K$ is almost an isometric operator on the column space of $Y_{<r}(\mX)$, one can show that $\E_\vx|T_K(f_{<r})(\vx)-f_{<r}(\vx)|^2= \Delta_d\to 0$ in probability, {\it pointwisely}. 
\paragraph{Critical-frequency.} Up to a vanishing term,
one can remove all non-critical frequencies in the kernel function $K$ in $T_K$ in the sense of making the following substitutions 
\begin{align}
    K(\vx, \mX) \to \sigma_r^2 Y_r(\vx)^\top Y_r(\mX)^\top
\quad\mathrm{and}\quad 
K_r(\mX, \mX') \to  \sigma_r^2 Y_r(\mX) Y_r(\mX)^\top + \hat h^2_{>r} \mI_m\,.
\end{align}
Thus, with $\gamma= (\lambda + \hat h_{>r}^2)$, $ T_{K}f_r(\vx) -f_r(\vx)= Y_r(\vx)^\top \mM_r(\mX) \vf_r + \Delta_d$, where  
\begin{align}
   \mM_r(\mX) = \left(\mI_{N(d, r)} - \sigma_r^2 Y_r(\mX)^\top
    (\sigma_r^2 Y_r(\mX) Y_r(\mX)^\top + \gamma\mI_m)^{-1}   Y_r(\mX) \right)\,.
\end{align}
Taking expectation with respect to $\vx$ (using orthogonality of $Y_r(\vx)$) and then with respect to $\vf_r$,
\begin{align}
\E_{\vf_r}   \E_\vx    |Y_r(\vx)^\top \mM_r(\mX) \vf_r|^2 = 
\E_{\vf_r} | \mM_r(\mX) \vf_r|^2
= \hat F_r^2 \mathrm{Tr}(\mM_r^2) / N(d, r)\,.
\end{align}
Applying the Sherman–Morrison–Woodbury formula and then Theorem \ref{Theorem:mp},  
\begin{align*}
  \hat F_r^2  \mathrm{Tr}\left( Y_r(\mX)^\top Y_r(\mX) /m +
  m\sigma_r^2/\gamma \mI_{N(d,r)}\right)^{-2} / N(d,r) \to
    \hat F_r^2 \int \frac{\mu_{\alpha}(t)}{(t+\xi_r)^{2}} dt\,.
    % =  \hat F_r^2 \chi_B(\alpha, \xi) 
\end{align*}
{\bf High-frequencies.} The cross term $\E_\vf\E_\vx \mT_{K}f_{>r}(\vx)  f_{>r}(\vx)=\Delta_d $ and thus
\begin{align}
     \E_{\vf, \vx}|\mT_{K}f_{>r}(\vx) - f_{>r}(\vx)|^2 = 
     \E_{\vf, \vx}|\mT_{K}f_{>r}(\vx)|^2+
     \E_{\vf, \vx} |f_{>r}(\vx)|^2 +\Delta_d\,.
\end{align}
The second term is equal to $\hat F_>^2$. The calculation of the first term is similar to that of the critical frequency above (namely, we remove all high-/low-frequency components in $K$.)

\section{Convolutional Kernels}
\label{sec:cnn}
\subsection{One hidden layer}
Our analysis can be extended to analyzing NNGP kernel and NT kernel for one-layer convolution~\citep{novak2018bayesian, novak2019neural, yang2020tensor}. In this case, we assume the input space is $\bm \gX = \sS_{d_0-1}^p$, where $d_0$ is the dimension of a patch, $p$ is the number of patches, and $d=pd_0$ is the total dimensions of the inputs.  
The measure associated to $\bm\gX$ is the product of the uniform measure on $\sS_{d_0-1}$. 
We assume that both the filter size and stride of the convolution are equal to $d_0$. As such, after the first convolutional layer, the input is reduced to a vector of dimension $p$. We then apply a non-linearity and a dense layer to map this $p$-dimensional vector to a scalar. The NNGP and NT kernel have the following general form. Let $\vx=(\vx_i)_{i\in [p]}\in \bm\gX$, where $\vx_i\in\sS_{d_0-1}$ is the $i$-th patch
\begin{align*}
    K(\vx, \vx') = \frac 1 p \sum_{i\in [p]}h(\vx_i^\top\vx_i')
    = \frac 1 p \sum_{i\in [p]}\sum_{k\geq 1}\hat h_k^2 P_k(\vx_i^\top\vx_i') 
    = 
    \sum_{k\geq 1}\frac { \sigma_k^2} p  \sum_{i\in [p]}\sum_{l\in [N(d_0,k)]}Y_{kl}(\vx_i)Y_{kl}(\vx_i')\,.
\end{align*}
Denote $Y_{kl}^{(i)}(\vx) = Y_{kl}(\vx_i)$ and $Y_k(\vx) = [Y_{kl}^{(i)}(\vx)^\top]_{l\in [N(d_0,k)], i\in [p]}^\top$. Then $Y_k(\vx)$ is the degree $k$ spherical harmonics associated to this kernel, which span a space of dimension $pN(d_0,k)$.
\begin{theorem}\label{Theorem:mp-conv}
Let $r\in \sN^*$ and $\alpha \in(0, \infty)$ be fixed. If $pN(d_0, r)/m\to \alpha\in (0,\infty)$ as $d_0\to\infty$ and the rows of $\mX$ are sampled uniformly, iid from $\sS_{d_0-1}^p$, then the empirical spectral distribution of $\frac 1m Y_r(\mX)^\top Y_r(\mX)$ tends to the Marchenko-Pastur distribution $\mu_\alpha$  as $d_0\to\infty$. 
\end{theorem}
The assumptions on the label function are similar to that of dot-product kernel, e.g.\footnote{The Gaussian assumption is unessential. We use it here for convenience.} 
\begin{align}
    f(\vx) = \sum_{k\geq 1}\vf_k^\top Y_k(\vx),\quad \mathrm{with}\quad \vf_k \sim \gN\left(0, \frac {\hat F_k^2}{pN(d_0, k)} \mI_{pN(d_0, k)}\right)\,\, \mathrm{if}\, k\geq r \,,
\end{align}
otherwise $\vf_k$ is deterministic with $\|\vf_k\|_2^2=\hat F_k^2$. 
\begin{theorem}\label{thm:generalization-one-layer-conv}
Let $\alpha\in (0, \infty)$ and $r\geq 1$ be fixed. Assume $pN(d_0,r)/m\to \alpha$ as $d_0\to\infty$. Then the average test error is given by 
\begin{align}\label{eq:generalization formula-conv}
    \mathrm{Err}(\mX;  \lambda, \vF, \hat\vh)= 
   \mathrm{B}_r(\alpha; \lambda,  \vf, \hat\vh) + \mathrm{V}_r(\alpha; \lambda, \vf, \hat\vh)  + \Delta_{d_0},   
\end{align}
where $\Delta_{d_0}\to 0$ in probability as $d_0\to\infty$. 
\end{theorem}

\begin{corollary}\label{cor:learning-curve-conv}
If, for $1\leq r \in \sN$, (1) $pN(d_0,r)/m \to\alpha$ for some $\alpha \in(0, \infty)$, or (2) $pd_0^{r-1+\delta}\lesssim m \lesssim pd_0^{r-\delta}$ for some $\delta\in(0, 1/2)$, then
\begin{align}
    \mathrm{Err}(\mX; \vf, \hat\vh, \lambda) = 
    \mathrm{LC}(m; \lambda, \vf, \hat\vh) + \Delta_{d_0}
\end{align}
where $\Delta_{d_0} \to 0$ in probability as $d_0\to\infty$. 
\end{corollary}

\subsection{Deep Convolutional Kernels}\label{sec:deep-conv-generalization}
The eigenstructure of general CNN kernels are much more complicated as they depend on both the frequencies (i.e. the order of the polynomials) and the topologies of the networks \citep{xiao2021eigenspace}. To rigorously describe the eigenstructure, a heavy dose of notation must be introduced, which is beyond the scope of the paper. Nevertheless, the approach developed here is readily extended to cover general CNN kernels. We briefly describe the main ideas.

Following \citep{xiao2021eigenspace}, 
we assume the input space is still $\bm \gX = \sS_{d_0-1}^p$, where $p$ is the number of patches. For simplicity, we assume the network has $L$ convolutional layers and in each layer, the filter size and the stride are all equal to $d_0$. Thus the spatial dimension of the input is reduced to 1 after $L$ convolutional layers. We then add a non-linearality and a dense layer to generate the logits. The kernel has the following form 
\begin{align}
    K(\vx, \vx') = \sum_{\vk\in \sN^{p}}\sum_{\vl\in \prod_{i\in [p]} [N(d_0, k_i)]} \sigma^2_{\vk, \vl} Y_{\vk,\vl}(\vx)Y_{\vk,\vl}(\vx'),\, \mathrm{where}  \quad 
    Y_{\vk,\vl}(\vx) =  \prod_{i\in [p]} Y_{k_i,l_i}(\vx_i)\,.
\end{align}
Unlike dot-product kernels in which $\vk$ is a scalar and the eigenvalues depend only on $|\vk|$ (i.e. the frequencies), $\sigma^2_{\vk, \vl}$ depends on both $|\vk|$ and the spatial structure of the vector $\vk$ in a rather complicated manner. Nevertheless, as $d_0\to\infty$, $\sigma^2_{\vk, \vl}\sim d_0^{-j_{\vk}} =d^{-j_\vk/L}$ for some $L\leq j_\vk\in\sN$. We can then categorize the eigenvectors according to the decay order of $\sigma^2_{\vk, \vl}$. Unlike the case of dot-product kernels or the one-hidden layer CNN kernels, in which eigenvectors with same-order eigenvalues are in the same eigenspace (i.e. the eigenvalues are the same), multiple-layer CNN kernels can have  {\it multiple} eigenspaces with the same-order eigenvalues. Although this results in extra challenges (see below), our overall approach carries over. Consider the critical scaling regime $m\sim d^r$, for $r=j/L$ for some $L\leq j\in\sN$. Likewise, we can decompose the kernel into low-, critical- and high-frequency parts according to $j_\vk<r$, $j_\vk=r$ and $j_\vk>r$, resp. Following similar assumptions on the labels and eigenvalues, the bias and the variance can be essentially reduced to computing 
\begin{align}\label{eq:conv-bias}
\chi_B =     &\frac 1 {N_r}\Tr \mR_r^2  (\mR_r +  Y_{r }(\mX)^\top Y_{r }(\mX)/m )^{-2}  \\
\label{eq:conv-variance}
\chi_V =    &\frac {N_{\leq r}}{m}\frac 1 {N_{\leq r}}\Tr   (\mR_{\leq r} +  Y_{\leq r }(\mX)^\top   Y_{\leq r }(\mX)/m )^{-2}   Y_{\leq r}(\mX)^\top   Y_{\leq r}(\mX)/m
\end{align}
where $\mR_r$ and $\mR_{\leq r}$ are diagonal matrices whose entries are determined by the eigenvalues of the critical-frequency modes. In the dot-product kernels or one-hidden layer CNN kernels setting, $\mR_r$/$\mR_{\leq r}$ is a scaled identity matrix and simple, closed-form expressions for the above traces straightforwardly follow from the Marcenko-Pastur distribution. However, for a general diagonal matrix $\mR_{r}$ with bounded limiting spectra, $\mR_{r}$ does not commute with $Y_r^\top(\mX)Y_r(\mX)$, and a more detailed random matrix analysis is needed. See the supplementary material for more details.

\section{Experiments}
\label{sec:experimental-setup}
We provide experiments to show that our learning curves (\myeqref{eq:learning-curves}) accurately capture empirical sample-wise learning curve even when the ambient dimensions remains small. Even though our theoretical results require averaging the test error over random labels (aka, mean test error), our experimental results suggest this is unnecessary, i.e. the learning curve \myeqref{eq:learning-curves} can capture the test error accurately for any given draw of label function.  
\paragraph{Experimental setup.} We generate a polynomial kernel function $h(t)=\sum_{k=1}^7 \hat h_k^2 P_k(t)$, where $P_k$ is the degree-$k$ Legendre polynomial in $d$ dimensions. The kernel function can be efficiently computed via $K(\vx, \vx') = h(\vx^\top\vx')$. We choose the label function to be $f(\vx) = \sum_{k=1}^7 \hat F_k y_k(\vx)$, where $y_k(\vx) = \sum_{j\in [d]}w_{k, j}\prod_{i=j}^{j+k-1}{x_i}$, and the coefficients $w_{k, j}$ are randomly sampled from a Gaussian and then normalized so that $\E_\vx|y_k(\vx)|^2=1$ for each $k$. Therefore $\E_\vx |f(\vx)|^2=\sum_{k=1}^7 \hat F_k^2$. For simplicity, we also set $\sigma_\epsilon^2 =0$ (i.e. noiseless) and $\lambda=0$ (i.e. ridgeless). Note that when $m \lesssim d^r$, the regressor still contains ``effective noise" $\hat F_{>r}^2$ from un-learnable high-frequency modes  and ``effective regularization" $\hat h_{>r}^2$. Finally, in our experiments, we choose $\hat F_k^2 = k^{-2}$ and $\hat h_k^2 = \mathrm{Gap}^{-(k-1)}$, where we will vary the value of the spectral gap: $\mathrm{Gap} = \hat h_k^2 /\hat h_{k+1}^2$. 
Under this setup, the predicted learning curve $\mathrm{LC}(m) = \mathrm{LC}(m; \mathrm{Gap})$ depends only on the spectral gap of the kernel. 

To simulate higher-order scaling ($r\geq 3$), 
the dimension $d$ has to be very small as we need to invert a sequence of matrices of size ranging from $m=1$ to $m \propto d^r/r!$.
Due to the constraints in compute and memory, the largest $m$ we can have is typically $m_{\mathrm{max}}\approx 25,000$ for one single GPU and $d$ in our experiments is typically around $d=24$. As such we are in a regime with strong finite-size corrections. Finally, all experiments are run in a single A-100 using Google Cloud Colab Notebooks.

\begin{figure}[t]
     \centering
     \begin{subfigure}[b]{0.45\textwidth}
         \centering
    \includegraphics[width=1.\textwidth]{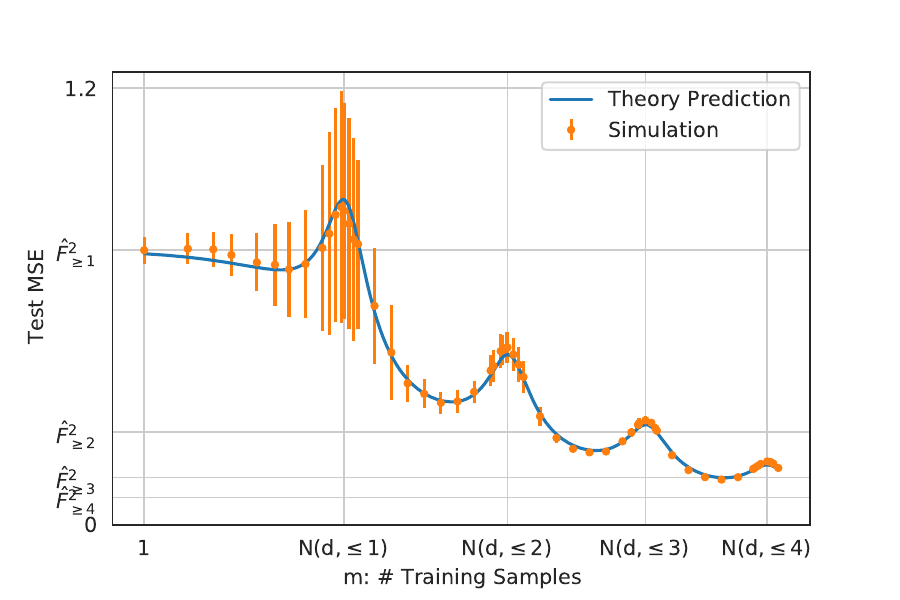}
         \caption{Dot Product Kernel}
         \label{fig:y equals x}
     \end{subfigure}
     \quad
     \begin{subfigure}[b]{0.45\textwidth}
         \centering
        \includegraphics[width=1.\textwidth]{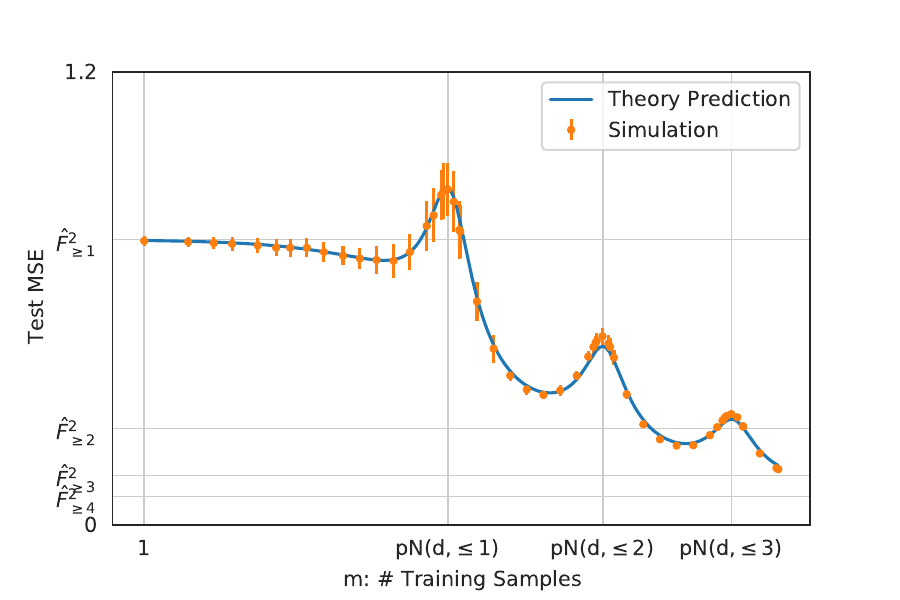}
         \caption{One-layer CNN Kernel}
         \label{fig:one-layer-cnn}
     \end{subfigure}
     \hfill
         \caption{{\bf Simulation vs Prediction.}
         We generate the learning curves obtain from {\bf\color{plot_orange} kernel regression} by densely varying $m$ from 1 to $24000$. For each $m$, we average the MSE over 20 runs. The {\bf\color{plot_blue}closed-form prediction} from \myeqref{eq:generalization formula} captures the simulations surprising well even for small $d$. Left: dot product kernel with $d=24$. Right: one-hidden layer CNN kernel with $d_0=20$ and $p=6$. The spectral gap is $\mathrm{Gap}=32$ in both plots.}
    \label{fig: fc-cnn-prediction}

\end{figure}

\paragraph{Learning Curves Accurately Capture Simulations.} In Fig.~\ref{fig: fc-cnn-prediction}, we generate the empirical sample-wise learning curve by applying kernel regression \myeqref{eq:kernel regression error} with training set $\mX$. We vary the training set size $m$ densely in $[1, m_{\mathrm{max}}]$ and for each $m$ we sample 20 independent $\mX$ to get the {\bf\color{plot_orange}errorbar plot} for the test error. The closed-form {\bf\color{plot_blue}learning curve} is obtained from \myeqref{eq:learning-curves} and the calculation is done in Sec.\ref{sec:computing integrals}. Even in the low-dimensional regime with $d=24$ for dot-product kernel ($d_0=20$ and $p=6$ for one-hidden layer CNN kernel), the predicted learning curve captures the empirical learning curve surprisingly well, which has a highly non-trivial multiple-descent behavior. It is worth mentioning that, from the simulation, the deviation of the test error from its mean is relatively large when $m$ is small but vanishes quickly as $m$ becomes larger. This suggests Theorem \ref{thm:generalization-fc-kernel} and Corollary \ref{cor:learning-curve} should hold in a pointwise fashion, i.e. without averaging the test error over random labels. 

\paragraph{Finite-size correction vanishes as $d\to\infty$.} 
Our theoretical results assume that the input dimension $d$ is sufficiently large and these results are exact when $d=\infty$. To visualize the finite-size correction, we plot the dependence of the correction (between simulations and predictions) on the the input dimension $d$.   Fig.~\ref{fig:finit-size-discrepancy} (a) shows that the means of simulations are converging to the theoretical prediction. Fig.~\ref{fig:finit-size-discrepancy} (b) shows that the standard deviations are converging zero.

\begin{figure}
    \centering
    \begin{subfigure}[b]{0.45\textwidth}
    \includegraphics[width=\textwidth]{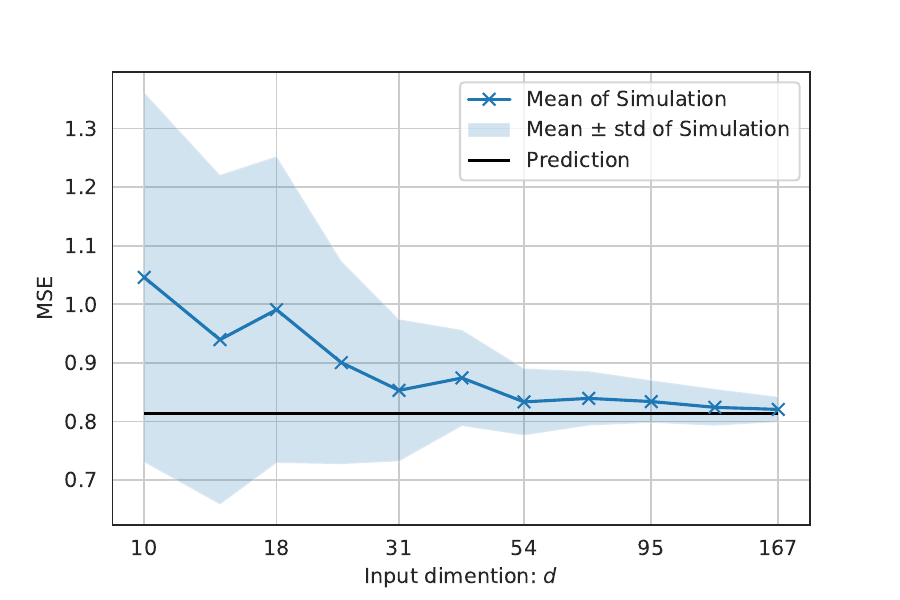}
    \caption{Mean of Simulations. }
    \end{subfigure}
    \begin{subfigure}[b]{0.45\textwidth}
    \includegraphics[width=\textwidth]{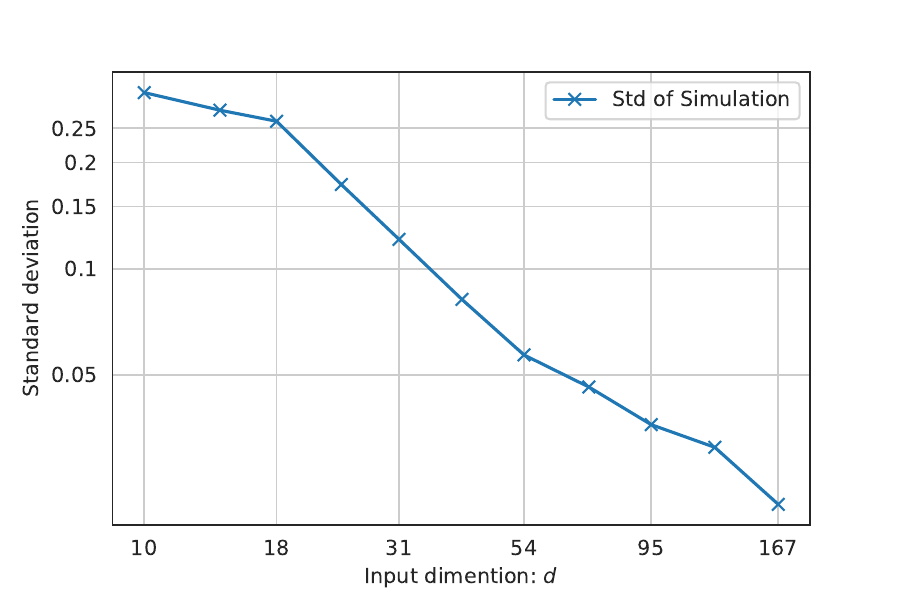}
    \caption{Standard Deviation of Simulations.}
    \end{subfigure} 
    \caption{{\bf Simulations approach predictions as $d\to\infty$.} The mean and the standard deviation are computed over 32 runs. The predictions and simulations are obtained via the second peak $r=2$, namely, $m=N(d, \leq 2)$.}
    \label{fig:finit-size-discrepancy}

\end{figure}

\paragraph{Small Spectral Gap Eliminates Multiple-descent.}
In Fig.~\ref{fig:multiple}, we plot both the predicted learning curves and simulations when $\mathrm{Gap}$ ranging in $[2, 8, 32, 128]$. For $1\leq r\leq 6$, we have 
$\xi_r = \mathrm{Gap}^{-(r-1)}/  \sum_{k=r}^7\mathrm{Gap}^{-(k-1)}$
and $\xi_r  \approx \mathrm{Gap} $ when $ \mathrm{Gap}$ is large. Recall that the variance term $\mathrm{V}_r$ peaks at $\alpha=1$ and the peak scales like $\xi_r^{1/2} \approx \mathrm{Gap}^{1/2}$. When $\mathrm{Gap}$ is large, e.g. $\mathrm{Gap}=32, 128$, the variance is also large near $\alpha=1$, the multiple-descent phenomena are more prominent. On the other hand, when $\mathrm{Gap}$ is small, e.g. $\mathrm{Gap}=8, 2$, such phenomena disappear and learning curves become monotonic.   
\section{Conclusion}\label{sec:conclusion}
In this work, we establish precise asymptotic formulas for the sample-wise learning curves in the kernel ridge regression setting for a family of dot-product kernels in the polynomial scaling regimes $m\propto d^r$ for all $r\in\sN^*$. We demonstrate that these formulas can capture empirical learning curves surprisingly well even in the regime where strong finite-size corrections would be expected. We rigorously prove that the learning curves can be non-monotonic near $m\propto d^{r}/r!$ for each $r\in\sN^*$. There are a couple limitations of our approach which could be improved in future work. The first one is the strong assumption on the distribution of the input data, namely, the uniform distribution on the spherical type of data. In addition, the learning curves are obtained only in the kernel regression setting and extending the results to the random feature setting (see, e.g., \cite{lu2022equivalence}) and the feature learning setting \cite{yang2020feature} would be meaningful future directions.

\section*{Acknowledgement}
We thank Ben Adlam for providing valuable feedback on a draft.
T.M. was supported by NSF through award DMS-2031883 and the Simons
Foundation through Award 814639 for the Collaboration on the Theoretical
Foundations of Deep Learning. T.M. also acknowledge the NSF grant
CCF-2006489 and the ONR grant N00014-18-1-2729.
The work of Yue M. Lu is supported by a Harvard FAS Dean’s competitive fund award for promising scholarship, and by the US National Science Foundation under grant CCF-1910410.

\bibliographystyle{plainnat}

\bibliography{main.bbl}

\newpage 
\appendix

\section{Appendix Guidelines}
The appendix is organized as follows. We prove Theorem \ref{Theorem:mp} and Theorem~\ref{Theorem:mp-conv} in Sec.~\ref{sec:proof-pm}. In  
Sec.~\ref{sec:proof-of-generalization}, we prove the test errors for dot-product kernels, namely, Theorem \ref{thm:generalization-fc-kernel} and for one-hidden-layer convolutional kernels, namely, Theorem \ref{thm:generalization-one-layer-conv}. 
The proof also shows how to reduce the test error of multiple-layer convolutional kernels to evaluating \myeqref{eq:conv-bias} and \myeqref{eq:conv-variance}. 
Finally, in Sec.\ref{sec:additional-plots}, we provide additional plots to empirically verify that the finite-size correction becomes smaller as $d$ grows larger.

\section{Proof of Theorem \ref{Theorem:mp}} \label{sec:proof-pm}

We begin with some notations. For positive numbers $a$ and $b$, we use $a\lesssim b$ to mean there is a constant independent of $d$ such that $a\leq C b$. In addition, $a\sim b$, if $a\lesssim b$ and $b\lesssim a$.

The proof of Theorem \ref{Theorem:mp-conv} is similar. We only present the proof of Theorem \ref{Theorem:mp}. 
Our proof is based on the following result from \cite{bai2008large}.
\begin{lemma}[\cite{bai2008large}]\label{Lemma:bai-zhou}
Let $\vx_p\in\mathbb R^{p}$ be random vectors and $ \mX = [\vx_{p1}, \dots \vx_{pn}]$ be a $p\times m$ matrix with iid columns. 
If for every $\{\mA_p\}_{p}$, $p\times p$ matrix with uniform operator norm,  
\begin{align}
\frac 1 {p^2} \E |\vx_p^T \mA_p \vx_p - {\rm Tr}(\mA_p) )|^2\to 0,
\end{align}
 then the empirical spectral distribution of $\frac 1 m \mX\mX^T$ converges to $\mu_\alpha$ weakly if $p/m\to\alpha\in(0, \infty)$. 
\end{lemma}

We prove a slightly more general version. 
\begin{theorem}\label{Theorem:mp-general}
Let $r\in \sN$ and $\alpha\in (0,\infty)$ be fixed. Assume $m=m(d)$ with $N(d, r)/m\to \alpha\in (0,\infty)$ as $d\to\infty$. 
Let $\vu=\vu(\vd)$ be a sequence of functions defined on $\sS_{d-1}$ such that \begin{enumerate}
    \item[(1.)] the cardinality of $\vu$ satisfies $|\vu|/d^r\to 0 $ as $d\to\infty$; 
    \item[(2.)] the functions in $\vu$ and $Y_{r}$ are mutually orthogonal; 
    \item[(3.)] for any unit vector $\theta$, let $\E_\vx |\theta^TZ_r(\vx)|^4\lesssim 1$ uniformly of $d$ and $\theta$, where $Z_r(\vx)$ is the concatenation of $\vu(\vx)$ and $Y_{r}(\vx)$. 
\end{enumerate}  
Let $Z_r(\mX)$ be the concatenation of $Y_r(\mX)$ and $\vu(\mX)$. Then the empirical spectral distribution of $\frac 1m Z_r(\mX)^T Z_r(\mX)$ converges in distribution to the Marchenko-Pastur distribution $\mu_{\alpha}$. 
\end{theorem}
 
We mainly use the case when $\vu$ is the empty set, i.e. $Z_r=Y_r$ and the case when $\vu=[Y_{kl}]_{k<r}^T$, i.e. $Z_r = Y_{\leq r}$.

\begin{proof}[Proof of Theorem \ref{Theorem:mp-general}]
We apply Lemma \ref{Lemma:bai-zhou} to $Z_r^T(\mX)$. We only need to show for matrices $\mA=\mA^{(d)}$ with $\|\mA\|_{op}\leq 1$, 
\begin{align}
(|\vu| +N(d, r))^{-2}\mathbb E_x |Z_r(\vx)^T \mA Z_r(\vx) - {\rm Tr}(\mA) |^2\to 0 \, \quad \mathrm{as} \quad d\to\infty\, , 
\end{align}
or, equivalently 
\begin{align}\label{eq:quadratic-form}
N(d, r)^{-2}\mathbb E_x |Z_r(\vx)^T \mA Z_r(\vx) - {\rm Tr}(\mA) |^2\to 0 \, \quad \mathrm{as} \quad d\to\infty. 
\end{align}
since $|\vu| = o(d^r)$. 
The assumption $\|\mA\|_{op}\leq 1$ implies that the absolute values of all entries of $\mA$ are bounded 1. 
A key observation in proving the above estimate is that, up to a unitary transformation, almost all functions in $\{Y_{r,l}(\vx)\}_l$ are monomials of the form 
\begin{align}
    g_\vi(\vx)  = C_{d, r}\prod_{i\in\vi} x_i
\end{align} where $\vi\subseteq [d]$ with $|\vi|=r$ and $C_{d, r}$ is a normalizing factor such that 
\begin{align}
   C_{d, r}^2 \int_{\sS_{d-1}} \prod_{i\in\vi}|x_i|^2 d\vx = 1 \, . 
\end{align}
We prove later that for any finite integer $r\geq 1$,
\begin{align}
    \int_{\sS_{d-1}} \prod_{i\in\vi}|x_i|^2 d\vx = 
    \prod_{i\in\vi} \int_{\sS_{d-1}} |x_i|^2 d\vx
    + O(d^{-r-1} ) = d^{-r} + O(d^{-r-1} ) 
\end{align}
Now we proceed to prove \myeqref{eq:quadratic-form}. 
First note that $\sG = \{g_\vi: \vi \subseteq [d], |\vi|=r\}$ is an orthonormal set. This can be proved by noticing that if $\vi\neq \vj$, then there is $i\in \vi$ but $i\notin \vj$. Clearly, the symmetries of the measure on $\sS_{d-1}$ implies 
\begin{align}
\int_{\sS_{d-1}}g_\vi(\vx) g_\vj(\vx) d\vx=0 \, .   
\end{align}
We choose $\sB = \{b_j\}_{j\in [p]}$ so that $\sG \sqcup \sB $ forms an orthonormal basis of $Z_r(\vx)$.
Note that the cardinality of $\sB$ is $o(d^r)$. Indeed, 
\begin{equation}
\begin{aligned}
    p =&~ |\vu|  + N(d, r)- {{d}\choose {r} } \\
    =&~ |\vu|  +  {{d+r-2}\choose {r} } + {{d+r-3}\choose {r-1} }-{{d}\choose {r} } = |\vu|  +  O(d^{r-1}) = o(d^r)
\end{aligned}
\end{equation}
Thus $\frac p {N(d, r)}\to 0$ as $d\to\infty$. After a change of basis, we can assume $Z_r(\vx) = [\vg(\vx)^T, \vb(\vx)^T]^T$, where $\vg = [g_\vi]_\vi^T$ and $\vb =[b_j]_{j\in [p]}^T $. Here we use the fact that \myeqref{eq:quadratic-form} holds for all $\mA$ with uniform operator norms is equivalent to that it holds for $\mQ^T\mA\mQ$ for all such $\mA$ and any unitary matrix $\mQ$. 
We write,   
\begin{align}
    \mA =  \begin{bmatrix}
    \mA_{11}& \mA_{12} 
    \\
    \mA_{21} & \mA_{22}
    \end{bmatrix}
\end{align}
where $\mA_{11}$ is the upper left ${{d}\choose{r}}\times {{d}\choose{r}}$ block of $\mA$, $\mA_{22}$ is the lower right $p\times p$ block of $\mA$ and the other two blocks are defined similarly. Note that $\|\mA_{ij}\|_{op}\leq \|\mA\|_{op}\leq 1$ for $i, j \in \{1, 2\}$. 
As such, we have 
\begin{align}
    N(d, r)^{-2}\E_\vx |Z_r(\vx)^T \mA Z_r(\vx) - {\rm Tr}(\mA) |^2
    \leq 5 (I_1 +I_2+I_3+I_4+I_5)
\end{align}
where
\begin{align}
I_1 &= N(d, r)^{-2} \E_\vx
    |\vg(\vx)^T \mA_{11} \vg(\vx) - {\rm Tr}(\mA_{11})|^2
    \\
    I_2 &= N(d, r)^{-2}
    \E_\vx |\vb(\vx)^T \mA_{22} \vb(\vx)|^2
    \\
    I_3 &= N(d, r)^{-2}\E_\vx 
    |\vg(\vx)^T \mA_{12} \vb(\vx)|^2
    \\ I_4 &= N(d, r)^{-2}\E_\vx 
    |\vb(\vx)^T \mA_{21} \vg(\vx)|^2
    \\
    I_5 &= N(d, r)^{-2} |{\rm Tr}(\mA_{22})|^2
\end{align}
We prove $I_i\to 0$ for $1\leq i \leq 5$. The $i=1$ case is the most difficult and the others are straightforward since $pN(d, r)^{-1}\to 0$. E.g., when $i=3$
\begin{align}
    I_3 &\leq 
    N(d, r)^{-2}\E_\vx 
    \|\vg(\vx)\|_{l_2}^2 \|\mA_{12}\|_{op}^2 \|\vb(\vx)\|^2_{l_2}
    \\&\leq  N(d, r)^{-2}
    \left(\E_\vx  \|\vg(\vx)\|_{l_2}^4  \E_\vx   \|\vb(\vx)\|^4_{l_2}\right)^{1/2}
    \\
    &\leq 
      N(d, r)^{-2} \max_{\vi, j} 
     (\E_\vx |g_\vi(\vx)|^4 \E_\vx |b_j(\vx)|^4) ^{1/2}
     p N(d, r)
     \\ &= Cp N(d, r)^{-1} \to 0\, 
\end{align}
where we have set $C = (\E_\vx |g_\vi(\vx)|^4 \E_\vx |b_j(\vx)|^4) ^{1/2}$, which is $O(1)$ due to assumption (3.) in Theorem \ref{Theorem:mp-general}. The bounds for $I_2$ and $I_4$ can be obtained similarly. 
For $I_5$, we simply use ${\rm Tr}(\mA_{22}) \leq p \|\mA\|_{op}\leq p$. 

It remains to control $I_1$. To ease the notation, denote $\mB= \mA_{11}$. We split $I_1\leq I_{11} + I_{12}$, where $I_{11}$ and $I_{12}$ are the diagonal and the off-diagonal parts, resp.,  
\begin{align}
    I_{11} &= 2 N(d, r)^{-2} \E_\vx | \sum_{\vi} B_{\vi\vi}(g_\vi^2(\vx) - 1)|^2 
    \\I_{12} &=
    2 N(d, r)^{-2} \E_\vx | \sum_{\vi\neq\vj} B_{\vi\vj}g_\vi g_\vj(\vx)|^2 
\end{align}
\paragraph{Bounding the diagonal part $I_{11}$.} Using $\E_\vx g_\vi(\vx)^2=1$, we have 
\begin{align}
  I_{11} =  2 N(d, r)^{-2} \E_\vx \sum_{\vi, \vj} B_{\vi\vi}B_{\vj\vj}(g_\vi^2(\vx) g_\vj^2(\vx) - 1)
\end{align}
We spit the proof into two cases: $\vi\cap \vj =\emptyset$ and $\vi\cap \vj \neq\emptyset$. 
The following is the key estimate to handle the first case. 
\begin{lemma}\label{lemma: spherical-estimate}
If $\vi\cap \vj =\emptyset$, 
\begin{align}
\max_{\vi, \vj} |\E_\vx g_\vi^2(\vx) g_\vj^2(\vx)-1| \lesssim d^{-1} \, .    
\end{align}
\end{lemma}
We prove this lemma later. 
We show how to use this lemma to handle the $\vi\cap\vj =\emptyset$ case. Recall that $|B_{\vi,\vj}|\leq 1$ and the number of tuples $(\vi, \vj)$ is fewer than $N(d, r)^2$. We have 
\begin{align}
    &N(d, r)^{-2} |\E_\vx \sum_{\vi, \vj, \vi\cap \vj=\emptyset} B_{\vi\vi}B_{\vj\vj}(g_\vi^2(\vx) g_\vj^2(\vx) - 1)|
    \\\leq  &
    N(d, r)^{-2} 
    \sum_{\vi, \vj, \vi\cap \vj=\emptyset} 
    \max_{\vi, \vj} |\E_\vx g_\vi^2(\vx) g_\vj^2(\vx) -1|
    \\ \leq  &\max_{\vi, \vj} |\E_\vx g_\vi^2(\vx) g_\vj^2(\vx) -1| \lesssim d^{-1} \,. 
\end{align}
We turn to $|\vi\cap \vj| = t$, $1\leq t\leq r$. For each fixed $\vi$, the number of choices of $\vj$ is 
\begin{align}
\sum_{1\leq t\leq r}{r\choose t } {d-r\choose  r-t }\lesssim  \sum_{1\leq t\leq r} d^{r-t}\sim d^{r-1}
\end{align}
As such,  
\begin{align}
    &2 N(d, r)^{-2} \E_\vx \sum_{\vi, \vj, \vi\cap\vj\neq \emptyset} |B_{\vi\vi}B_{\vj\vj}(g_\vi^2(\vx) g_\vj^2(\vx) - 1)|
    \\
    \lesssim  &  N(d, r)^{-2} N(d, r) d^{r-1} \max_{\vi, \vj }\E_\vx |g_\vi^2(\vx) g_\vj^2(\vx) - 1|
    \\
    \lesssim & d^{-1}
\end{align}

\paragraph{Off-diagonal terms $I_{12}$.}
Bounding the off-diagonal terms can be reduced to a combinatorics problem, which is similar to the random tensor model considered in  \cite{bryson2021marchenko}. We need to estimate 
\begin{align}
  I_{12} =   2 N(d, r)^{-2}  \sum_{\vi\neq\vj, 
  \vl\neq\vk} B_{\vi\vj} B_{\vl\vk}
  \E_\vx g_\vi(\vx) g_\vj(\vx)
  g_\vl(\vx) g_\vk(\vx)
\end{align}
By symmetries of the uniform measure on the sphere, we can assume the monomial given by $g_\vi(\vx) g_\vj(\vx)
  g_\vl(\vx) g_\vk(\vx)$ has no linear factor, that is the degree of any $x_i$ in this monomial must be at least $2$ if not 0. In addition, for such monomials, the Holder inequality and hypercontractivities yield, 
  \begin{align}
      |\E_\vx g_\vi(\vx) g_\vj(\vx)
  g_\vl(\vx) g_\vk(\vx)| 
  \leq \max_{\vi}\E_\vx |g_\vi(\vx)|^4
  \lesssim \max_{\vi} (\E_\vx |g_\vi(\vx)|^2 )^2 = 1 \,.  
  \end{align}
 As such, we only need to show that the growth of the number of such quadruples $(\vi, \vj, \vk,\vl)$, as a function of $d$, is slower than $N(d, r)^2\sim d^{2r}$.  
 We proceed to prove this claim. For each fixed $\vi$, let $t = |\vi \cap \vj|$ where $0\leq t \leq r-1$ ($t\neq r$ since $\vi\neq \vj$). Let $J(\vi;t)$ denote the set of such $\vj$, whose cardinality is 
\begin{align}
    |J(\vi;t)| = {r \choose t} {d-r \choose r-t} <   r^{t} d^{r-t} \lesssim d^{r-t}\,. 
\end{align}
Next we estimate the number of tuples $(\vl, \vk)$. Let $w = |(\vl\cup\vk) \backslash (\vi\cup\vj)|$. Since $|\vl\cup\vk \cup \vi\cup\vj|\leq 2r$ and $|\vi\cup\vj|=2r-t$, we have $w\leq t$. The cardinality of choosing such $\vk\cup \vl$ cannot exceed
\begin{align}
    \sum_{0\leq w\leq t}{d - (2r-t)\choose w}\sum_{v=0}^{2r-w}{2r\choose v} \lesssim d^{t} \, .
\end{align}
With $\vk\cup \vl$ given, the pair of $(\vk, \vl)$ cannot exceed ${2r \choose r}^2\lesssim 1$. Thus, with $\vi, \vj$ and $t$ fixed, the number of pairs of $(\vk, \vl)$ is $\lesssim d^{t}$. Using $|B_{\vl \vk}| \leq 1$ and $\max_{\vi} (\sum_{\vj}B_{\vi\vj}^2)^{1/2}\leq \|\mB\|_{op}\leq 1$, we have
\begin{align}
   N(d, r)^{2}   I_{12} &\lesssim \sum_{\vi} \sum_{0\leq t\leq r-1}\sum_{j\in J(\vi; t)} B_{\vi\vj} d^{t}
    \\&\leq 
    \sum_{\vi} \sum_{0\leq t\leq r-1} | J(\vi; t)|^{1/2}
    (\sum_{\vj\in J(\vi; t)} B_{\vi\vj}^2)^{1/2} d^{t}
    \\&\lesssim
    \|\mA\|_{op} \sum_{\vi} \sum_{0\leq t\leq r-1} d^{(r-t)/2}
    d^{t}
    \\&\lesssim N(d, r) d^{2r-1/2}
\end{align}
which gives $I_{12} \lesssim d^{-1/2}$.
\end{proof}

\begin{proof}[Proof of Lemma \ref{lemma: spherical-estimate}]
It suffices to prove that for any finite integer $j>1$,
\begin{align}\label{eq:induction-1}
\displaystyle
{\int_{\sS_{d-1}} \prod_{1\leq t\leq j} x_t^2d\vx}  =d^{-j} + O(d^{-j-1}).  
\end{align}
Indeed, assuming this estimate, we have 
$C_{d, j}^{-2} = d^{-j} + O(d^{-j-1})$ and 
$C_{d, j}^{2} = d^{j} + O(d^{j-1})$.  
For any $\vi$ and $\vj$ with $\vi\cap \vj = \emptyset$, 
\begin{align}
    \E_\vx g_\vi^2(\vx) g_\vj^2(\vx)-1 
    = C_{d, r}^4\int_{\sS_{d-1}} \prod_{t\in \vi\cup\vj}x_t^2d\vx
    =  C_{d, r}^4 C_{d, 2r}^{-2} -1 = O(d^{-1}) 
\end{align}
It remains to prove \myeqref{eq:induction-1}. 
By symmetries,  
\begin{align}
    \int_{\sS_{d-1}} x_t^2 d\vx = \frac 1 d 
    \int_{\sS_{d-1}} \sum_{1\leq t\leq d}x_t^2 d\vx
    =\frac 1 d
     \int_{\sS_{d-1}} 1 d\vx
    =\frac 1 d  \, .
\end{align}
By symmetries again,   
\begin{align}
    &\int_{\sS_{d-1}} \prod_{1\leq t\leq j-1} x_t^2 d\vx
    \\=&
    \int_{\sS_{d-1}} \prod_{1\leq t\leq j-1} x_t^2 \left(\sum_{1\leq i \leq j-1} x_i^2 + \sum_{j\leq i \leq d} x_i^2 \right) d\vx
    \\=&
     (d-j+1) 
    \int_{\sS_{d-1}} \prod_{1\leq t\leq j} x_t^2 d\vx +
     (j-1)  \int_{\sS_{d-1}} x_1^2 \prod_{1\leq t\leq j-1} x_t^2 d\vx
    %  \\=&
    % \int_{\sS_{d-1}} \prod_{1\leq t\leq j} x_t^2 d\vx +
    %  \frac {j-1} d  \int_{\sS_{d-1}} (x_1^2 -x_j^2) \prod_{1\leq t\leq j-1} x_t^2 d\vx 
\end{align}
We use hypercontractivities to bound the error term, namely, the second term. Recall that for any $q\geq 2$ and any polynomial defined on the sphere, 
\begin{align}
     \left(\int_{\sS_{d-1}} |f(\vx)|^q d\vx \right)^{1/q}\leq  (q-1)^{\deg(f)/2}  \left(\int_{\sS_{d-1}} |f(\vx)|^2 d\vx \right)^{1/2} \,. 
\end{align}
Setting $f(\vx)=x_t$ (with $\deg(f)=1$) gives 
\begin{align}
    \int_{\sS_{d-1}} |x_t|^q d\vx \leq  (q-1)^{q/2}  \left(\int_{\sS_{d-1}} |x_t|^2 d\vx \right)^{q/2}
    =  (q-1)^{q/2} d^{-\frac q 2} \, . 
\end{align}
By Holder's inequality and symmtries 
\begin{align}
    \int_{\sS_{d-1}} x_1^2 \prod_{1\leq t\leq j-1} x_t^2 d\vx
    &\leq \left(
    \int_{\sS_{d-1}} |x_1|^{2j} d\vx 
    \prod_{1\leq t\leq j-1} \int_{\sS_{d-1}} |x_t|^{2j} d\vx \right)^{\frac 1 j} \\&= \int_{\sS_{d-1}} |x_1|^{2j} d\vx 
    \leq (2j-1)^{j} d^{-j}
\end{align}
Thus 
\begin{align}
    \int_{\sS_{d-1}} \prod_{1\leq t\leq j-1} x_t^2 d\vx = 
    (d-j+1) 
    \int_{\sS_{d-1}} \prod_{1\leq t\leq j} x_t^2 d\vx + O(d^{-j})
\end{align}
and 
\begin{align}
    \int_{\sS_{d-1}} \prod_{1\leq t\leq j} x_t^2 d\vx =& 
    \frac{1}{d-(j-1)}\int_{\sS_{d-1}} \prod_{1\leq t\leq j-1} x_t^2 d\vx
     + O(d^{-j-1})  \, . 
    % \\=&
    %  \prod_{1\leq t\leq j} \int_{\sS_{d-1}} x_t^2 d\vx
    %  + C_j d^{-j-1}
\end{align}
Finally, \myeqref{eq:induction-1} is a consequence of this estimate and induction.

\end{proof}

\section{Generalization}\label{sec:proof-of-generalization}
We aim to obtain the asymptotic formulas for the test error in this section. 
In the high-level, we decompose the empirical kernel $K(\mX, \mX)$ into low-, critical- and high-frequency modes, 
where we have concentration in the low- and high-frequency parts of the kernel. The test error associated to these two parts are easier to handle. 
The critical-frequency part is more difficult in which random matrix behaviors emerge, namely, the Marchenko-Pastur distribution.  
As such, our first step is to remove the contribution in the test error coming from the non-critical frequency parts. After that, the remaining is essentially equivalent to computing the trace of certain functional forms related to the Marchenko-Pastur distribution. 

We consider a general setting that includes the dot-product kernels, the one-hidden-layer and the multiple-layer convolutional kernels (NNGP and NT kernels.) 
In what follows, we use $\Delta_d, \Delta_d', \Delta_d''$, etc. to represent quantities that converge to 0 in probability (the absolute value of a scalar, the norm of a vector, the operator norm of a matrix, etc.), whose exact form may change from line to line.

\subsection{Setup}

For $d\in\sN^*$, let $\bm \gX^\upd\subseteq \sR^d$ be the input space associated with a probability measure $\sigma^\upd$ and a kernel function $K^\upd$. Assume the kernel function has the following eigen-structure  
\begin{align}\label{eq:general-eigen-decomposition}
    K^\upd(\vx, \vx') = \sum_{k\geq 1} \sum_{n\in [E_k]} (\sigma_{kn}^\upd)^2 \sum_{l\in N_{kn}^\upd} \phi^\upd_{knl}(\vx)\phi^\upd_{knl}(\vx')
\end{align}
in the sense $ K^\upd$, as the integral operator from $L^2(\bm\gX^\upd, \sigma^\upd)$ to itself, 
\begin{align}
   K^\upd \phi_{knl}^\upd(\vx) = \int  K^\upd(\vx, \vx')\phi^\upd_{knl}(\vx') \sigma^\upd(d \vx') = (\sigma_{kn}^\upd)^2 \phi^\upd_{knl}(\vx). 
\end{align}
Here $\{\phi^\upd_{knl}\}_{knl}$ is an orthonormal basis of $L^2(\bm\gX^\upd, \sigma^\upd)$. We also assume $ K^\upd$ is a trace-class operator, i.e., 
\begin{align}
   \sum_{knl}\langle K^\upd \phi_{knl}^\upd , \phi_{knl}^\upd \rangle =  \sum_{k\geq 1} \sum_{n\in [E_k]} N_{kn}^\upd (\sigma_{kn}^\upd)^2   <\infty\,.
\end{align}
In the above notations, we use the triplet $(k, n, l)$ to index the eigenfunctions $\phi_{knl}^\upd$. The tuple $(k, n)$ determines the eigenspace, whose eigenvalue is of the form ``$(\sigma_{kn}^\upd)^2 = C_nd^{-s_k} + \mathrm{Lower\,\, Order}$" and $l$ lists all eigenfunctions in the $kn$-eigenspace. 
We make the following assumptions. 
\paragraph{Kernel Assumptions.} 
\begin{enumerate}
\item[(1.)] {\bf Spectral Gap.} There are  $\delta_0>0$ and a sequence of strictly increasing positive real numbers $\{s_k\}$ with $|s_k - s_{k-1}|\geq \delta_0$  for all $k\geq 2$ such that 
\begin{align}
    (\sigma_{kn}^\upd)^2 \sim d^{-s_k} \sim (N_{kn}^\upd)^{-1}
\end{align}
Moreover, $\{E_k\}\subseteq \sN^*$ is independent from $d$ which grows at most exponentially. We also assume that there is a sequence of real numbers $\{\hat h_{kn}^2 \}_{kn}$ with $\hat h_{kn}^2\neq 0$ unless $k$ is sufficiently large and 
\begin{align}
   \sum_{k}\sum_{n\in [E_k]}\hat h_{kn}^2 <\infty \quad \mathrm{and}\quad (\sigma_{kn}^\upd)^2 N_{kn}^\upd = \hat h_{kn}^2 
   \quad \mathrm{as}\quad d\to\infty 
\end{align}
    \item[(2.)] {\bf Hypercontractivity Inequalities.} For any $p\geq 2$ there are constant $C_{p,k}$ such that for any function $f$ in the closure of $\mathrm{Span}\{\phi^\upd_{jnl}\}_{j\leq k}$
    \begin{align}
        \|f\|_p \leq C_{p,k} \|f\|_2
    \end{align}
    \item[(3.)] {\bf Concentration of Quadratic Forms.} 
    Let $\phi^\upd_k(\vx)$ denote the column vector consists of elements $\{\phi^\upd_{knl}(\vx)\}_{l\in[N_{kn}\upd n\in[E_k]}$. 
    For every sequence of matrices $\{\mA^\upd\}$ with uniformly bounded operator norm, 
    \begin{align}
       \left(\sum_{n\in [E_k]} N_{kn}^\upd\right)^{-2} \E_{\vx}| \phi^\upd_k(\vx)^\top  \mA^\upd \phi^\upd_k(\vx)  -\Tr\mA^\upd|^2\to 0 
       \quad \mathrm{as}\quad d\to\infty . 
    \end{align}
    \item[(4.)]{\bf Addition Theorem.}  For $k\in 
    \sN^* $ and $n\in [E_k]$ and $\vx\in\bm\gX^\upd$ 
    \begin{align}
        \sum_{l\in [N_{kn}^\upd]} \phi_{knl}^\upd (\vx)^2 = N_{kn}^\upd
    \end{align}
\end{enumerate}
Let us briefly explain the assumptions. The {\bf Spectral Gap} assumption basically says, we can classify the eigenvectors into countably many categories indexed by $k\in\sN^*$. In the $k$-th category, it has exactly $E_k$ many eigenspaces, each of them has dimensions $\sim d^{s_k}$ and eigenvalues $d^{-s_k}$. It also implies the number of eigenfunctions with eigenvalues $\lesssim d^{-s_k}$ is $\sim d^{s_k}$.
Assumptions (1.), (2.) and (4.) together are stronger than those in Theorem 6 in \citet{mei2021generalization} (and slightly less technical), which allow us to apply kernel concentration from \citet{mei2021generalization}. In particular, they imply concentration of the low- and high-frequency parts of the empirical kernel $K^\upd(\mX, \mX)$. Finally,
Assumption (3) is designed to meet the requirements in Lemma \ref{Lemma:bai-zhou}, which allows us to claim Marchenko-Pastur type behavior of the gram matrix induced by the feature map $\phi_k$. We provide a couple examples. 
\begin{example}[Dot-product Kernels]
When $\bm\gX^\upd = \sS_{d-1}$ and $K^\upd$ is the dot-product kernel, we have $E_k=1$, $s_k=k$, $N_{kn}=N(d, k)  \sim d^k/k!$, and $\phi_{knl}=Y_{kl}$ (note that $n=0$ since $E_k=1$.) Note that by the Addition Theorem of spherical harmonics (Theorem 4.11 in \citet{frye2012spherical}),  
\begin{align}
\sum_{l\in [N(d, k)]}    Y_{kl}(\vx)^2 = N(d, k)
\end{align}
\end{example}
\begin{example}[One-hidden-layer Convolutional Kernels]
Sightly more general setting is the one-layer convolutional kernel (NNGP or NT kernels). In this case, $\bm\gX^\upd = \sS_{d_0-1}^p$ where $p$ is the number of patches and the input dimension is $d=pd_0$. We can set either $p=O(1)$ (i.e. independent of $d_0\to\infty$) or $p\sim d^{\alpha_p}$ for some $\alpha_p>0$. This kernel is essentially the sum of $p$ dot-product kernels. As such, $E_k=1$, $N_{kn}\upd = pN(d_0, k)\sim pd_0^k/k!$
and $(\sigma_{kn}^upd)^2 \sim (pd_0^{k})^{-1}$. If $p\sim d^{\alpha_p}$ and $d_0\sim d^{\alpha_{d_0}}$ with $\alpha_{d_0}+ \alpha_p=1$, we have $s_k = \alpha_p + k\alpha_{d_0}$ and $d^{-s_k}$ is the decay rate of the $k$-th order spherical harmonics. 
\end{example}

\begin{example}[Multiple-layer Convolutional Kernels]
General convolutional kernels are much more complicated \citep{xiao2021eigenspace}. In this case,  $\bm\gX^\upd = \sS_{d_0-1}^p$ where $p$ is the number of patches and the input dimension is $d=pd_0$. 
We additionally assume, $p=k_0^{L-1}$ for some $k_0\in\mathbb N^*$ and the network has $L$ convolutional layers with filter size and strides being the same in each layer (equal to $d_0$ in the first layer and to $k_0$ for the remaining $(L-1)$ layers.)
The eigenstructures of such kernels are studied in \citet{xiao2021eigenspace}. The eigenfunctions are tensor products of spherical harmonics defined on copies of $\sS_{d_0-1}$, 
\begin{align}
    Y_{\vk, \vl}(\vx) = \prod_{i\in [p]}Y_{k_il_i}(\vx_i)
\end{align}
The eigenvalues are more complicated to compute as they depend on both the frequencies of $Y_{\vk,\vl}$ and the topologies of the networks. When $d_0\propto d^{\alpha_{d_0}}$ and $k_0\propto d^{\alpha_{k_0}}$ with $\alpha_{d_0} + (L-1)\alpha_{k_0}=1$ and $\alpha_{k_0}, \alpha_{d_0}>0$, the eigenvalue of $Y_{\vk\vl}$ is $\propto d^{-(\frequency(\vk) +\spatial(\vk) ) }$, as $d\to\infty$. Here $\frequency(\vk)\equiv |\vk|\alpha_{d_0}$ is the frequency index of $Y_{\vk \vl}$ and $\spatial(\vk) = J_\vk\alpha_{k_0}$ is the spatial index, where $J_\vk$ is the number of edges in the sub-tree connecting all interacting patches (i.e. $k_i\neq 0$) to the output; see \citet{xiao2021eigenspace} for more details.

As there can possibly exist $\vk$ and $\vk'$ with $\frequency(\vk)+ \spatial(\vk)= \frequency(\vk') + \spatial(\vk')$ (i.e., same order of decay) but $(\frequency(\vk), \spatial(\vk))\neq (\frequency(\vk'),\spatial(\vk'))$ (i.e. different space-frequency combination), 
there can exist more than one eigenspaces whose eigenvalues decay to zero with the same rate $d^{-(\frequency(\vk) +\spatial(\vk) )}$, but with different leading coefficients. This is the main reason why we need to allow $|E_k|>1$ in   \myeqref{eq:general-eigen-decomposition}. 
\end{example}

Next we discuss the assumptions on the label function. 
Let $\mX$ be the training set
with $m\sim d^{s_r}$ many training samples for some $r\in \sN^*$ fixed. Then let the ground true label function to be 
\begin{align}
f(\vx) = \sum_{k\in \sN^*} \sum_{n\in [E_k]} \sum_{l\in [N_{kn}^\upd]}
\hat f_{knl }\phi_{knl}^\upd(\vx). 
\end{align}
Let  $ N_k^\upd = \sum_{n\in [E_k]} N_{kn}^\upd
$. We assume, for $k\geq r$, $\hat\vf_{kn} = \{\hat f_{knl}\}_{ l\in [N_{kn}^\upd]}$ is a random vector with 
\begin{align}
    \E \hat\vf_{kn} = \0   \quad  \mathrm{and}\quad  \E \hat\vf_{kn}\hat\vf_{kn}^\top  = \frac{\hat F_{kn}^2} {N_{kn}^\upd } \mI_{N_{kn}^\upd}. 
\end{align}
and $\{\hat\vf_{kn}\}_{n\in[E_k], k\geq r}$ are mutually independent. One concrete example is 
\begin{align}
   \hat \vf_{kn} \sim  \mathcal N\left(\0,   \frac{\hat F_{kn}^2}{N_{kn}^\upd} \mI_{N_{ kn}^\upd}\right)\,.
\end{align}

For $k<r$, we assume the coefficients are deterministic with $\sum_{l}\hat f_{knl}^2 = \hat F_{kn}^2$, $\sum_n\hat F_{kn}^2 = \hat F_k^2$ and   
\begin{align}
\sum_{k\in \sN^*}    \hat F_{k}^2 <\infty
\end{align}
Our goal is to compute the average test error over the random labels defined above in the scaling limit $m\sim d^{s_r}$. 

\subsection{Structure of the Empirical Kernels}
For convenience, denote  
\begin{align}
    \phi^\upd_{\leq k}(\vx) &= [\phi^\upd_{jnl}(\vx)]_{l\in [N_{jn}^\upd], 1\leq j\leq k, n\in [E_j], }^\top 
    \\ 
    N_k^\upd &= \sum_{n\in [E_k]} N_{kn}^\upd 
    \\
    N_{\leq k}^\upd &= 
    \sum_{1\leq j\leq k}N_{\leq j}^\upd
\end{align}
Let $\vx_i^\top $ be the $i$-th row of the training matrix $\mX$. 
Similarly, 
\begin{align}
    &Z_k(\mX) = [\phi^\upd_k(\vx_0), \dots, \phi^\upd_k(\vx_{m-1})]^\top  \quad
    \quad 
    && Z_{\leq k} (\mX) = [\phi^\upd_{\leq k}(\vx_0), \dots, \phi^\upd_{\leq k}(\vx_{m-1})]^\top 
    \\
    &\mLambda_k = \mathrm{diag}\left( [(\sigma_{kn}^\upd )^2 \mI_{N_{kn}^\upd} ]_{n\in [E_k]}\right)
    &&
    \mLambda_{
    \leq k } = \mathrm{diag} \left([(\sigma_{jn}^\upd )^2\mI_{N_{jn}^\upd}]_{n\in [E_j], 1\leq j\leq k}\right)
\end{align}
Note that $Z_k(\mX)$ ($Z_{\leq k}(\mX)$) is an $m\times N_{k}^\upd$ ($m\times N_{\leq k}^\upd$) matrix and $\mLambda_k$ ($\mLambda_{\leq k}$) is an $N_k^\upd \times N_k^\upd$ ($N_{\leq k}^\upd \times N_{\leq k}^\upd$) diagonal matrix. The followings are defined similarly, 
\begin{align}
    Z_{<k}(\mX),\quad Z_{kn}(\mX),\quad \mLambda_{<k}, \quad
    \mLambda_{kn},\quad  N_{<k}^\upd,\quad N_{kn}^\upd   \, .
\end{align}

Next, we decompose the train-train kernel into two parts: the $\leq r$ frequency part and the $>r$ frequency parts,  
\begin{align}
    K^\upd(\mX, \mX) &= \sum_{k\in\sN^*} Z_k(\mX)\mLambda_k Z_k(\mX)^\top  
    = Z_{\leq r}(\mX) \mLambda_{\leq}  Z_{\leq r}(\mX)^\top  + \sum_{k\geq r+1} Z_k(\mX)\mLambda_k  Z_k(\mX)^\top   
    \\
    &= Z_{\leq r}(\mX) \mLambda_{\leq}  Z_{\leq r}(\mX)^\top  + \sum_{k\geq r+1}\sum_{n\in [E_k]} 
    (\sigma_{kn}^\upd)^2Z_{kn}(\mX)Z_{kn}(\mX)^\top   
    \\&\equiv 
    K_{\leq r}^\upd(\mX, \mX) + K^\upd_{> r}(\mX, \mX) 
\end{align}
{\bf Assumptions (1.) (2.)} allow us to apply kernel concentration \citep{ghorbani2021linearized, mei2021generalization}, which implies that the low-frequency and high-frequency parts of the empirical kernels are concentrated. By saying concentration in the high-frequency part, we mean 
\begin{claim}\label{claim-tail-estimate}
Let 
\begin{align}
    \Delta_{kn}^\upd \equiv \frac 1 {N_{kn}^\upd} Z_{kn}(\mX) Z_{kn}(\mX)^\top  - \mI_m \,.
\end{align}
Then 
\begin{align}\label{eq:tail-control}
    \E \sum_{k>r}\sum_{n\in [E_k]} \|\Delta_{kn}^\upd\|_{\op} \to 0 
\end{align}
\end{claim}
The proof of this claim essentially follows from the arguments and results in Theorem 6 of \citep{mei2021generalization}; see Sec.~\ref{sec:proof-of-claim-tail}. 
Thus 
\begin{align}
    K^\upd_{> r}(\mX, \mX)   &= 
    \sum_{k\geq r+1}\sum_{n\in[E_k]} (\sigma_{kn}^\upd)^2 N_{kn}^\upd (\displaystyle \mI_{m} + \Delta_{kn}^\upd)
    \\&= \sum_{k\geq r+1} \hat h_k^2 \mI_m 
    + \sum_{k\geq r+1}\sum_{n\in [E_k]} \hat h_{kn}^2 \Delta_{kn}^\upd
\end{align}
Denote 
\begin{align}
    \hat h_{>r}^2 &= \sum_{k\geq r+1}\sum_{n\in [E_k]} N_{kn}^\upd(\sigma_{kn}^\upd)^2 \sim 1
    \\
    \Delta_{>r}^\upd &= \sum_{k\geq r+1}\sum_{n\in [E_k]} \hat h_{kn}^2 \Delta_{kn}^\upd \, , 
\end{align} we can write
\begin{align}\label{eq:high-concentration}
K^\upd_{> r}(\mX, \mX)  = 
 \hat h_{>r}^2 \mI_m + \Delta_{>r}^\upd, 
 \quad \mathrm{where} \quad \E \|\Delta_{>r}^\upd\|_{\op}\to 0 
\end{align}
By saying the low-frequency part of the kernel concentrates, we mean (Theorem 6 \citep{mei2021generalization})
\begin{align}\label{eq:low-concentration}
    \frac 1 m Z_<(\mX)^\top  Z_<(\mX) = \mI_{N_{<}^\upd} + \Delta_{<r}^\upd  , 
    \quad \mathrm{where} \quad \E  \|\Delta_{<r}^\upd\|_{\op}\to 0 
\end{align}
Finally, Lemma \ref{Lemma:bai-zhou} and {\bf Assumption (3.)} imply that if $N_{r}^\upd /m\to\alpha\in(0, \infty)$, then the empirical measure of the critical part of the kernel matrix  
$
\frac 1 m Z_r(\mX)^\top Z_r(\mX) 
$
, and the low-and-critical frequency parts 
$
\frac 1 m Z_{\leq r}(\mX)^\top Z_{\leq r}(\mX) 
$
converge to the Marchenko-Pastur distribution $\mu_\alpha$ weakly by Lemma \ref{Lemma:bai-zhou}. In particular, $\|\frac 1 m Z_{ r}(\mX)^\top Z_{ r}(\mX) \|_\op + \|\frac 1 m Z_{\leq r}(\mX)^\top Z_{\leq r}(\mX) \|_\op =O(1)$ in probability as $d\to\infty$.

For convenience, we summarize the structure of the empirical kernel in the following. 
\begin{corollary}\label{corollary:kernel-structure}
Assume {\bf Assumptions (1.-4.)}. Let $r\in\sN^*$ and $\alpha>0$ be fixed and $m=m^\upd$ be such that $N_{r}^\upd /m\to\alpha\in(0, \infty)$ as $d\to\infty$. Let $\mX$, of shape $m\times d$, be the training set matrix whose rows are drawn, uniformly, iid from $\bm \gX^\upd$. Then we have the following structure for the empirical kernel matrix 
\begin{align}
    &\text{\bf High-frequency Features: } \quad\quad&&K^\upd_{> r}(\mX, \mX)  = 
 \hat h_{>r}^2 \mI_m + \Delta_{>r}^\upd, 
 \\
 &\text{\bf Low-frequency Features: }
  &&\frac 1 m Z_<(\mX)^\top  Z_<(\mX) = \mI_{N_{<}^\upd} + \Delta_{<r}^\upd
\\
 &\text{\bf Critical-frequency Features: }
&&\text{the empirical measure of}\, \, 
\frac 1 m Z_r(\mX)^\top Z_r(\mX) \to \mu_\alpha
\\
 &\text{\bf Low-and-critical-frequency Features: }
&&\text{the empirical measure of}\, \, 
\frac 1 m Z_{\leq r}(\mX)^\top Z_{\leq r}(\mX) \to \mu_\alpha
\end{align}
\end{corollary}

 Let $0\leq \lambda = O(1)$ be the regularization and $\gamma = \lambda + \hat h_{>r}^2$ be the effective regularization. To ease the notations, denote  
\begin{align*}
&\overline \mZ_{<} = \frac 1 {\sqrt {m}}    Z_{< r}(\mX) 
\quad 
&&\overline \mZ_{\leq} = \frac 1 {\sqrt {m}}    Z_{\leq r}(\mX) 
\quad 
&&&\overline \mZ_{>} = \frac 1 {\sqrt {m}}    Z_{> r}(\mX)
\quad  
&&&&\overline \mZ_{r} = \frac 1 {\sqrt {m}}    Z_{r}(\mX)
\\
&\overline \mLambda_{<} =   \gamma^{-1} m \mLambda_{<}  
\quad 
&&\overline \mLambda_{\leq} =   \gamma^{-1} m \mLambda_{\leq}  
\quad 
&&&\overline \mLambda_{>} =   \gamma^{-1} m \mLambda_{>} 
\quad 
&&&&\overline \mLambda_{r} =   \gamma^{-1} m \mLambda_{r} 
\end{align*}
Clearly, Corollary \ref{corollary:kernel-structure} and the assumption on the spectra imply that in probability as $d\to\infty$, 
\begin{align}
    \|\overline \mZ_{<}\|_\op + 
    \|\overline \mZ_{\leq}\|_{\op} + 
     \|\overline \mZ_{r}\|_{\op} + 
    \|\overline \mLambda_{<} \|_\op+
    \|\overline \mLambda_{\leq } \|_\op+
    \|\overline \mLambda_{r} \|_\op \lesssim 1 
\end{align}
Then we can write $K^\upd$ and $K_\lambda^\upd$ as 
\begin{align}
K^\upd_{\lambda}(\mX, \mX) \equiv K^\upd(\mX, \mX) +\lambda \mI_m = \gamma   ( \overline \mZ_{\leq} \overline \mLambda_{\leq} \overline \mZ_{\leq}^\top   + \mI_m) 
+\Delta_{>r}^\upd \equiv \mK +\Delta_{>r}^\upd
\end{align}
where
\begin{align}
    \mK = \gamma   ( \overline \mZ_{\leq} \overline \mLambda_{\leq} \overline \mZ_{\leq}^\top   + \mI_m) 
\end{align}
Then by Sherman–Morrison–Woodbury formula 
\begin{align}\label{eq:K-formula}
    \mK^{-1} &= \gamma^{-1} \left( \mI_m -  \overline \mZ_{\leq} \left(  \overline \mLambda_{\leq} ^{-1}  + \overline \mZ_{\leq}^\top  \overline \mZ_{\leq} \right)^{-1}  \overline \mZ_{\leq}^\top   \right)  
    =
    \gamma^{-1} \left( \mI_m -  \overline \mZ_{\leq} \overline \mD^{-1}  \overline \mZ_{\leq}^\top   \right)
\end{align}
where 
\begin{align}
    \overline \mD &=  \overline \mLambda_{\leq} ^{-1}  + \overline \mZ_{\leq}^\top  \overline \mZ_{\leq} \, .
\end{align}
The matrix $\overline \mD$ plays a critical role in the remaining analysis. 
We have the following control regarding its eigenvalues, which says the eigenvalues of $\overline \mD$ are away from $0$ and $\infty$ 
\begin{lemma}\label{lemma:inv}
There are constants $0<\lambda_1 < \lambda_2$ independent of $d$ such that, in probability as $d\to\infty$, 
\begin{align}
    \lambda_1 \mI \prec \overline \mD \prec \lambda_2 \mI 
\end{align}

\end{lemma} 
We will prove the lemma later in Sec.\ref{sec:proof-lemma-key}. 

Note that  
\begin{align}\label{eq:k-inverse-estimate}
K^\upd_{\lambda}(\mX, \mX)^{-1}
 = 
(K^\upd_{\lambda}(\mX, \mX)^{-1}\mK)\mK^{-1} = 
(\mI_m + \mK^{-1} \Delta_{>r}^\upd)^{-1}\mK^{-1}
= (\mI_m + \Delta_d')\mK^{-1}
\end{align}
where $\|\Delta_d'\|_\op\to 0 $ in probability since $\|\mK^{-1}\|_\op\leq \gamma^{-1}$ and
$\|\Delta_{>r}^\upd\|_{\op}\to 0$ in probability. 

Similarly, we can write 
\begin{align}\label{eq:M-formula-2}
    M(\mX, \mX) \equiv& \E K(\mX, x) K(\mX, x)^\top  \\
    =&
    \sum_{1\leq  k\leq  r} \sum_{n\in [E_k]} (\sigma_{kn}^\upd)^4 Z_k(\mX)Z_k(\mX)^\top 
    +
    \sum_{k > r}\sum_{n\in [E_k]} (\sigma_{kn}^\upd)^4Z_{kn}(\mX)Z_{kn}(\mX)^\top 
    \\=& m^{-1} \gamma^{2} \overline \mZ_{\leq} \overline \mLambda_{\leq}^2 \overline \mZ_{\leq}^\top   + 
    m^{-1} \Delta_d''
\end{align}
where 
\begin{align}
\Delta_d'' \equiv m
\sum_{k > r}\sum_{n\in [E_k]} (\sigma_{kn}^\upd)^4Z_{kn}(\mX)Z_{kn}(\mX)^\top 
= 
m\sum_{k > r}\sum_{n\in [E_k]} (\sigma_{kn}^\upd)^4N_{kn}^\upd(\mI_m + \Delta_{kn}^\upd)
\end{align}
We have $\|\Delta_d'' \|_\op\to0$ in probability as $d\to\infty$, since  
\begin{align}
    m \sum_{k > r}\sum_{n\in [E_k]} (\sigma_{kn}^\upd)^4N_{kn}^\upd
    \sim & m \sum_{k > r} d^{-s_k}\sum_{n\in [E_k]} (\sigma_{kn}^\upd)^2N_{kn}^\upd
    \\
    \lesssim & d^{s_r} d^{-s_{r+1}}  \sum_{k > r}\sum_{n\in [E_k]} (\sigma_{kn}^\upd)^2N_{kn}^\upd
    \\
    \leq & d^{s_r} d^{-s_{r+1}}  \sum_{k>r}\sum_{n\in [E_k]} \hat h_{kn}^2
    \lesssim d^{-\delta_0}\to 0 
\end{align}
Finally, the above estimates imply
\begin{align}
   \mH &\equiv K_\gamma^\upd(\mX, \mX)^{-1} M(\mX, \mX) K_\gamma^\upd(\mX, \mX)^{-1}
   \\&= (\mI_m + \Delta_d') \mK^{-1} M(\mX, \mX) \mK^{-1} (\mI_m + \Delta_d')
   \\&= 
   m^{-1} (\mI_m + \Delta_d') \mK^{-1}(\gamma^{2} \overline \mZ_{\leq} \overline \mLambda_{\leq}^2 \overline \mZ_{\leq}^\top   + 
   \Delta_d'') \mK^{-1}(\mI_m + \Delta_d') 
   \\&=m^{-1}\left(\overline \mZ_{\leq} \left(  \overline \mLambda_{\leq} ^{-1}  + \overline \mZ_{\leq}^\top  \overline \mZ_{\leq} \right)^{-2} \overline \mZ_{\leq}^{T} +\Delta_d'''\right)
   \\&=m^{-1}\left(\overline \mZ_{\leq} \overline \mD^{-2} \overline \mZ_{\leq}^{T} + \Delta_d'''\right)
\end{align}
with the error term $\|\Delta_d'''\|_\op\to 0$ in probability. 

\subsection{Reduction I: Reducing the MSE to Traces}
We will repeatedly use the following simple results. 
\begin{lemma}
Let $\vu$ be a random vector with $\E \vu=\0$ and $\E\vu\vu^\top  = \sigma^2 \mI_k$. Then for any $k\times k$ deterministic matrix $\mA$, 
\begin{align}
    \E_\vu \vu^\top  \mA \vu &= \sigma^2  \Tr(\mA) 
    \\ 
   \E_\vu\|\mA\vu\|_2^2 &=\sigma^2  \Tr(\mA^\top \mA) \,. 
\end{align}
\end{lemma}

Next, we compute the loss by decomposing it as follows. Recall that the observed labels is $f(\mX) + \bm\epsilon$ where $\bm\epsilon$ is the iid noise term, which is centered and has variance $\sigma_\epsilon^2$. Thus the average test error is 
\begin{align*}
    &\mathrm{Err}(\mX;\lambda,  \vF,  \vh) \\=& \E_{\vf, \bm\epsilon, \vx}  \left|f(\vx) - K^\upd(\vx, \mX) K^\upd_\gamma(\mX, \mX)^{-1}(f(\mX)+\bm \epsilon)\right|^2
    \\
    =&\E_\vf \E_\vx f^2(\vx) - 2 \E_{\vf}\E_{\vx} f(\vx)K^\upd(\vx, \mX) K^\upd_\gamma(\mX, \mX)^{-1}f(\mX) +  \E_{\vf } f^\top (\mX) \mH f(\mX)  + \sigma_\epsilon^2\mathrm{Tr}(\mH)
    \\\equiv& 
    \E_{\vf} I_1 + \E_{\vf} I_2 + \E_{\vf} I_3 + I_4. 
\end{align*}
Here 
\begin{align}
    I_1&=\E_\vx f^2(\vx)
    \\
    I_2&=  - 2 \E_{\vf}\E_{\vx} f(\vx)K^\upd(\vx, \mX) K^\upd_\gamma(\mX, \mX)^{-1}f(\mX)
    \\
    I_3&= \E_{\vf } f^\top (\mX) \mH f(\mX)
    \\
    I_4 &=\sigma_\epsilon^2\mathrm{Tr}(\mH)
\end{align}

We estimate each $I_i$ individually.  

\paragraph{ Estimate $I_1$.}  We simply keep it unchanged at the moment.

\paragraph{ Estimate $I_2$.}  Note that 
\begin{align}
\E_\vx f(\vx)K^\upd(\vx, \mX) &= \sum_{k}\sum_{ n, l} \hat f_{knl} (\sigma_{kn}^\upd)^2 \phi^\upd_{knl}(\mX)^\top   =
    \sqrt m  \hat \vf_{\leq}^\top  \mLambda_{\leq} \overline \mZ_{\leq} ^\top 
    +
\sum_{k>r}    \sqrt m \hat \vf_{k}^\top  \mLambda_{k} \overline \mZ_{k} ^\top 
    \\
    &=
    \gamma \frac 1 {\sqrt{m}} \hat \vf_{\leq}^\top  \overline\mLambda_{\leq} \overline \mZ_{\leq} ^\top 
    +\sum_{k>r}  
   \gamma \frac 1 {\sqrt{m}}  \hat \vf_{k}^\top  \overline\mLambda_{k} \overline \mZ_{k} ^\top 
\end{align}
where $\hat \vf_\leq$ is the column vector with elements $\{\hat f_{knl}\}_{k\leq r}$ and $\hat \vf_<$,  $\hat \vf_>$, $\hat \vf_k$, etc. are defined similarly. 
In addition,  
\begin{align}
    f(\mX) = \sqrt m \overline \mZ_{\leq }\hat \vf_{\leq} + \sqrt m \sum_{k>r}\overline \mZ_{k }\hat \vf_{k}
\end{align}
We then use the fact that $\vf_>$ is centered to eliminate all cross terms between $\hat\vf_\leq$ and $\hat\vf_>$. Denote $\E_>$, $\E_k$ and $\E_{kn}$ the expectation operator over $\hat \vf_>$, over $\hat \vf_k$ and over $\hat \vf_{kn}$ resp. Under this notation, $\E_\vf= \E_r\E_>$. 
Using \myeqref{eq:k-inverse-estimate} and \myeqref{eq:M-formula-2}, 
\begin{align}
    \E_> I_2 = 
    -2  
    \hat \vf_{\leq}^\top 
    \overline \mD^{-1} & \overline \mZ_{\leq}^\top  \overline \mZ_{\leq}  
    \hat \vf_{\leq} - 
    2 \gamma \E_>  \sum_{k>r} \hat \vf_{k}^\top &\left( 
    \overline\mLambda_{k} \overline \mZ_{k} ^\top   K^\upd_\gamma(\mX, \mX)^{-1} \overline \mZ_{k} 
    \right) \hat \vf_{k} +\Delta_d
\end{align}
for some $\Delta_d\to 0$ in probability. 
The second term goes to zero since, for each $k>r$ and $n\in E_k$
\begin{align*}
&\E_{kn} \hat \vf_{kn}^\top \left( 
  \overline\mLambda_{kn} \overline \mZ_{kn} ^\top   K^\upd_\gamma(\mX, \mX)^{-1} \overline \mZ_{kn} 
    \right) \hat \vf_{kn}
    \\=
    & \hat F_{kn}^2/N_{kn}^\upd  \Tr \overline\mLambda_{kn}  \left( 
     \overline \mZ_{kn} ^\top   K^\upd_\gamma(\mX, \mX)^{-1} \overline \mZ_{kn} 
    \right) 
    \\= &  \hat F_{kn}^2/N_{kn}^\upd  \gamma^{-1}m(\sigma_{kn}^\upd)^2
    \Tr \left( 
     \overline \mZ_{kn} ^\top   K^\upd_\gamma(\mX, \mX)^{-1} \overline \mZ_{kn}\right)
     \\\leq &  \gamma^{-1} \hat F_{kn}^2/N_{kn}^\upd  m(\sigma_{kn}^\upd)^2
     \| K^\upd_\gamma(\mX, \mX)^{-1}\|_\op
    \Tr \left( 
     \overline \mZ_{kn}\overline \mZ_{kn} ^\top   \right)
     \\
     \lesssim& \hat F_{kn}^2/N_{kn}^\upd m (\sigma_{kn}^\upd)^2 \Tr (\frac 1m N_{kn}^\upd\frac 1{ N_{kn}^\upd }  \mZ_{kn} \mZ_{kn} ^\top   ) 
     \\=&
     \hat F_{kn}^2 m(\sigma_{kn}^\upd)^2 \frac 1 m \Tr(\mI_m+ \Delta_{kn}^\upd)
     \\\lesssim &
     m(\sigma_{kn}^\upd)^2 (1 + \|\Delta_{kn}^\upd\|_\op)
     \\=& m/N_{kn}^{\upd}\hat h_{kn}^2(1 + \|\Delta_{kn}^\upd\|_\op)
     \\\sim & d^{s_r-s_k} \hat h_{kn}^2(1 + \|\Delta_{kn}^\upd\|_\op)
    %  \\\sim & (\sigma_{kn}^\upd)^2 \lesssim m d^{-s_k}\sim d^{s_r-s_k}
\end{align*}
 Clearly, the sum over $k>r$ and $n\in [E_k]$ of the above is bounded by $\lesssim d^{-\delta_0}$ in probability. 
Thus 
\begin{align}
     \E_> I_2 = 
    -2  \left(
    \hat \vf_{\leq}^\top 
    \overline \mD^{-1} \overline \mZ_{\leq}^\top  \overline \mZ_{\leq}  
    \hat \vf_{\leq}
    \right) + \Delta_d
\end{align}

\paragraph{Estimate $I_3$.} Again, we use the fact that cross terms have mean zero  
\begin{align}
    \E_> I_3 & = f^\top (\mX) \mH f(\mX) =   m \left(  \hat\vf_{\leq}^\top   \overline \mZ_{\leq }^\top  \mH  \overline \mZ_{\leq }  \hat\vf_{\leq}
     + \sum_{k>r}\E_k  \hat\vf_k^\top   \overline \mZ_k^\top  \mH \overline \mZ_k  \hat\vf_k
     \right) 
     \\
     &= \hat \vf_{\leq}^\top   \overline \mZ_{\leq }^\top   \overline \mZ_{\leq }  \overline \mD^{-2} \overline \mZ_{\leq }^\top   \overline \mZ_{\leq }  \hat \vf_{\leq} 
     +  \sum_{k>r}\E_k \hat \vf_k^\top   \overline \mZ_k^\top  \overline \mZ_{\leq }  \overline \mD^{-2}  
     \overline \mZ_{\leq}^\top   \overline \mZ_k \hat \vf_k+\Delta_d
     \\
     &\equiv I_{3, 1} + I_{3, 2} +\Delta_d
\end{align}
Note that
\begin{align}
     I_{3, 2} &= 
      \sum_{k>r}\sum_n \E_{kn} \Tr  \overline \mZ_{\leq } \overline \mD^{-2}   \overline \mZ_{\leq }^\top \hat  \vf_{kn}  \vf_{kn}^\top   \overline \mZ_{kn} \overline \mZ_{kn}^\top 
      \\&=
      \sum_{k>r}\sum_n  \hat F_{kn}^2 / N_k^\upd \Tr  \overline \mZ_{\leq } \overline \mD^{-2}   \overline \mZ_{\leq }^\top   \overline \mZ_{kn} \overline \mZ_{kn}^\top 
      \\
      &=
      \sum_{k>r}\sum_n  \hat F_{kn}^2 /m \Tr  \overline \mZ_{\leq } \overline \mD^{-2}   \overline \mZ_{\leq }^\top  (\mI_m + \Delta_{kn}^\upd)
      \\&= 
      m^{-1}\sum_{k>r}\sum_{n} \hat F_{kn}^2 \left( \Tr \overline \mD^{-2}   \overline \mZ_{\leq }^\top   \overline \mZ_{\leq }\right) (1  + \Delta_d) 
      \\&=
       m^{-1} \hat F_{>r}^2 \Tr \overline \mD^{-2}   \overline \mZ_{\leq }^\top   \overline \mZ_{\leq } + \Delta_d' 
\end{align}
where $\Delta_d'\to 0$ in probability since $\|\mD^{-2}   \overline \mZ_{\leq }^\top   \overline \mZ_{\leq }\|_\op\lesssim 1$ in probability. 
\paragraph{Estimate $I_4$. }
The $I_4$ is similar to $I_{3,2}$ above and we have  
\begin{align}
    I_4 = m^{-1}  \sigma_\epsilon^2 \Tr \overline \mD^{-2}   \overline \mZ_{\leq }^\top   \overline \mZ_{\leq } + \Delta_d
\end{align}
\paragraph{All Together.} Combining all terms we have
\begin{align}
    &\E_> I_1 +I_2 +I_3 +I_4 \\&= \left\|\left(\mI- \overline \mD^{-1} \overline \mZ_{\leq}^\top  \overline \mZ_{\leq} \right) \hat \vf_\leq \right\|_2^2 
    +\left( 1 + m^{-1} \Tr \overline \mD^{-2}   \overline \mZ_{\leq }^\top   \overline \mZ_{\leq }\right) \hat F_{>r}^2 +  m^{-1} \Tr \overline \mD^{-2}   \overline \mZ_{\leq }^\top   \overline \mZ_{\leq } \sigma_\epsilon^2 +  \Delta_d 
    \\
    &= \left\|\overline \mD^{-1} \overline \mLambda_{\leq} ^{-1} \hat \vf_\leq \right\|_2^2 
    +\left( 1 + m^{-1} \Tr \overline \mD^{-2}   \overline \mZ_{\leq }^\top   \overline \mZ_{\leq }\right)  \hat F_{>r}^2 +  m^{-1} \Tr \overline \mD^{-2}   \overline \mZ_{\leq }^\top   \overline \mZ_{\leq } \sigma_\epsilon^2 + \Delta_{d}  
    \\&\equiv 
    T_1 + T_2 + T_3 + \Delta_{d}  
\end{align}
where 
\begin{align}
    T_1 &= \left\|\overline \mD^{-1} \overline \mLambda_{\leq} ^{-1} \hat \vf_\leq \right\|_2^2 
    \\T_2 &= \left( 1 + m^{-1} \Tr \overline \mD^{-2}   \overline \mZ_{\leq }^\top   \overline \mZ_{\leq }\right)  \hat F_{>r}^2 
    \\T_3 &= m^{-1} \Tr \overline \mD^{-2}   \overline \mZ_{\leq }^\top   \overline \mZ_{\leq } \sigma_\epsilon^2
\end{align}
As such, it remains to handle  
\begin{align}
    \E_r (T_1 +  T_2 +T_3) = \E_r T_1 +   T_2 + T_3 ,. 
\end{align}
\subsection{Reduction II: Reducing Traces to Integrals}

Recall that $\overline \mLambda_{\leq}$ is a diagonal matrix with elements $\gamma^{-1} m(\sigma_{kn}^\upd)^2$, whose multiplicity is $N_{kn}^\upd$ for $k\leq r$ and $n\in [E_k]$. When $k=r$, $\gamma^{-1} m (\sigma_{rn}^\upd)^2\sim 1$, otherwise $\gamma^{-1} m(\sigma_{kn}^\upd)^2\sim d^{s_r-s_k}$. 
For convenience, we let 
\begin{align}
    N_<\equiv N_<^{(d)}, \quad  N_=  \equiv N_r^{(d)}, \quad 
    N_\leq \equiv N_\leq^{(d)}
\end{align}

Therefore, 
\begin{align}
    \overline \mLambda_{\leq}^{-1} = 
    \begin{bmatrix}
    \Delta &0 
    \\
    0 & \mR 
    % (\gamma^{-1} m\sigma_r^2)^{-1} \mI_{N_=}
    \end{bmatrix}
\end{align}
where $\Delta$ is an $N_{<}\times N_<$ diagonal matrix whose entries are 
$\gamma m^{-1}(\sigma_{kn}^\upd)^{-2}\sim d^{-(s_r - s_k)}$, 
and $\mR$ is a $N_=\times N_=$ diagonal matrix whose entries are 
$\gamma m^{-1}(\sigma_{rn}^\upd)^{-2}\sim 1$. 
As such, we claim that we can replace $\hat \vf_\leq$ by 
$[\bm 0, \hat \vf_=]$ in estimating $T_1$ and replace the $\Delta$ in $ \overline \mLambda_{\leq}^{-1} $ by any $\rho\mI_{N_<}$ for any finite non-negative constant $\rho$ in estimating $T_2$. The first claim is obvious as, by Lemma \ref{lemma:inv}, 
\begin{align}
    \|\overline \mD^{-1} \overline \mLambda^{-1}_\leq [\bm \vf_<, \0]^\top  \|_2
    \leq 
    \|\overline \mD^{-1}\|_\op \|\overline \mLambda^{-1}_\leq [\bm \vf_<, \0]^\top  \|_2\lesssim \lambda_{1}^{-1}
    (m\sigma_{r-1}^2\gamma^{-1})
    \|\bm \vf_<\|_2 \lesssim d^{-(s_r - s_k)}\to 0 
\end{align}
To prove the second claim regarding estimating $T_2$, denote 
\begin{align}
    \tilde \mD =
     \begin{bmatrix}
    \rho \mI_{N_<} &0
    \\
    0 & \mR 
    % (\gamma^{-1} m\sigma_r^2)^{-1} \mI_{N_=}
    \end{bmatrix}
    + \overline \mZ_{\leq}^\top  \overline \mZ_{\leq}
\end{align}
We claim that, in probability,  
\begin{align}
    &m^{-1} \Tr\left( (\overline \mD^{-2} - \tilde \mD^{-2}  ) \overline \mZ_{\leq }^\top   \overline \mZ_{\leq }\right)
    \\=&
    m^{-1} \Tr\left( (\overline \mD^{-1}(\overline \mD^{-1} - \tilde \mD^{-1})+(\overline \mD^{-1} - \tilde \mD^{-1})  \tilde \mD^{-1} ) \overline \mZ_{\leq }^\top   \overline \mZ_{\leq }\right) \to 0
\end{align}
We only bound the first term as the second term can be handled similarly. 
\begin{align}
    &m^{-1} \Tr\left( \overline \mD^{-1}(\overline \mD^{-1} - \tilde \mD^{-1})\overline \mZ_{\leq }^\top   \overline \mZ_{\leq } \right) 
    \\=& 
     m^{-1} \Tr\left( (\overline \mD^{-1} - \tilde \mD^{-1})\overline \mZ_{\leq }^\top   \overline \mZ_{\leq } \overline \mD^{-1} \right)
     \\=& 
      m^{-1} \Tr\left( \overline \mD^{-1}(\tilde \mD -  \overline \mD) \tilde \mD^{-1}\overline \mZ_{\leq }^\top   \overline \mZ_{\leq } \overline \mD^{-1} \right)
      \\=& 
      m^{-1} \Tr\left( (\tilde \mD -  \overline \mD) \tilde \mD^{-1}\overline \mZ_{\leq }^\top   \overline \mZ_{\leq } \overline \mD^{-2} \right)
\end{align}
Then we use the facts that (1) the upper right $N_<\times N_<$ block matrix of $(\tilde \mD -  \overline \mD)$ is a diagonal matrix whose entries are in $[0, 1]$ and the three remaining block matrices are zeros, and (2) all entries in $\tilde \mD^{-1}\overline \mZ_{\leq }^\top   \overline \mZ_{\leq } \overline \mD^{-2}$ are bounded above by a constant (each matrix in $\tilde \mD^{-1} \overline \mZ_{\leq }^\top   \overline \mZ_{\leq } \overline \mD^{-2}$ has bounded operator norm\footnote{Recall that $\overline \mZ_{\leq }^\top   \overline \mZ_{\leq }$ follows the Marchenko-Pastur distribution.}) to conclude that 
\begin{align}
\left|m^{-1} \Tr\left( \overline \mD^{-1}(\overline \mD^{-1} - \tilde \mD^{-1})\overline \mZ_{\leq }^\top   \overline \mZ_{\leq } \right)\right| 
    \lesssim m^{-1} N_{<}
\end{align}
Thus 
\begin{align}
T_2 = 
\left( 1 + m^{-1} \Tr \tilde \mD^{-2}   \overline \mZ_{\leq }^\top   \overline \mZ_{\leq }\right) \hat  F_{>r}^2 + \Delta_d
\end{align}
which will be handled later. 

It remains to handle $T_1$. We make two steps of reductions in estimating $\E_r T_1$. The first one is to replace $ \overline \mLambda_{\leq}^{-1} $ by
\begin{align}
   \tilde  \mLambda_{\leq}^{-1} = 
    \begin{bmatrix}
    0 &0 
    \\
    0 & 
    \mR
    % (\gamma^{-1} m\sigma_r^2)^{-1}
    \end{bmatrix}
\end{align}
and the second one is to replace $\overline \mLambda_{\leq}^{-1} +  \overline \mZ_{\leq }^\top   \overline \mZ_{\leq }$ by
\begin{align}
\mW \equiv 
\begin{bmatrix}
\mI_{N_<}& \mB
\\
\mB^\top  & \mC 
\end{bmatrix}
\equiv 
    \begin{bmatrix}
     \mI_{N_<}& \overline \mZ_{< }^\top  \overline \mZ_{= }
     \\
      \overline \mZ_{= }^\top \overline \mZ_{< } & \overline \mZ_{= }^\top \overline \mZ_{= } + \mR 
    \end{bmatrix}
\end{align}
Here we have applied  
\begin{align}
    \overline \mZ_<^\top  \overline \mZ_< = \mI_{N_<} + \Delta_d 
\end{align}
The reason we could do so is exactly the same as we replaced $\overline \mD$ by $\tilde\mD$ above as we only perturb the entries in the upper $N_<\times N_<$ block by $O(1)$. 

Note that $\mW$ is symmetric and is also strictly positive definite, i.e. the minimal eigenvalue of $\mW$ is $\gtrsim 1$; see the proof in Sec.\ref{sec:proof-lemma-key}. 
Thus by the Schur complement, 
\begin{align}
\E_r    T_1 &= 
    \left\|\begin{bmatrix}
\mI_{N_<}& \mB  \\ 
\mB^\top  & \mC 
\end{bmatrix}^{-1}
 \begin{bmatrix}
    0 &0 
    \\
    0 & 
    \mR
    \end{bmatrix}
\begin{bmatrix}
0\\\hat \vf_{=}
\end{bmatrix}\right\|_2^2 + \Delta_d
\\
 &= \E_r
  \left\|
  \begin{bmatrix}
 -\mB (\mC - \mB^\top  \mB)^{-1}\mR\hat\vf_{=}
\\
(\mC - \mB^\top  \mB)^{-1}\mR\hat\vf_{=}
\end{bmatrix}
\right\|_2^2 + \Delta_d 
\end{align}
By the fact that $\hat\vf_=$ is mean zero and isotropic, we have the above equal to 
\begin{align}
  \E_r  T_1 & = \Tr
    \left ( 
    \mR(\mC - \mB^\top  \mB)^{-1} \mB^\top \mB (\mC - \mB^\top  \mB)^{-1}\mR
    + 
    \mR(\mC - \mB^\top  \mB)^{-2}\mR 
    \right)\hat F_r^2 / N_= + \Delta_d
    \\&= \Tr(\mR\mC^{-2}\mR)
   \hat F_{r}^2 / N_= + \Delta_d ' + \Delta_d
\end{align}
where   
\begin{align}
   \Delta_d' = \Tr
    &\left (\mR(\mC - \mB^\top  \mB)^{-1} \mB^\top \mB (\mC - \mB^\top  \mB)^{-1}\mR\right)/ N_= + \\ 
   & \Tr \left ( \mR((\mC - \mB^\top  \mB)^{-2} -  
    \mC ^{-2})\mR
    \right) /N_= \,. 
\end{align}
We claim that $\Delta_d'\to 0$ in probability. For the first term, we have 
\begin{align}
    &\Tr
    \left (\mR(\mC - \mB^\top  \mB)^{-1} \mB^\top \mB (\mC - \mB^\top  \mB)^{-1}\mR\right)/ N_=
    \\=&
    \Tr
    \left ((\mC - \mB^\top  \mB)^{-1}\mR^2(\mC - \mB^\top  \mB)^{-1} \mB^\top \mB \right)/ N_=
    \\=&\|(\mC - \mB^\top  \mB)^{-1}\mR^2(\mC - \mB^\top  \mB)^{-1}\|_\op
    \Tr
    \left (\mB^\top \mB \right)/ N_=
    \\\leq& \|(\mC - \mB^\top  \mB)^{-1}\|_\op \|\mR^2\|_\op 
    \|(\mC - \mB^\top  \mB)^{-1}\|_\op \Tr
    \left (\mB^\top \mB \right)/ N_=
    \\\lesssim& \Tr
    \left (\mB^\top \mB \right)/ N_= \sim N_</N_= \to 0 \, 
\end{align}
in probability as $d\to\infty$. We have used 
\begin{align}
\|\mR\|_\op\lesssim  1 
\\
    \|(\mC - \mB^\top  \mB)^{-1}\|_{\op} \leq \|\mW^{-1}\|_\op \lesssim 1 
    \\
    \frac 1 {N_=}\Tr\left (\mB^\top \mB \right) \sim N_</N_=\,. 
\end{align}
The last one holds because $\mB^\top \mB$ is a rank $N_<$ matrix with operator norm $\lesssim 1$. Note that this also implies that $\mB^\top \mB$ has at most $N_<$ many non-zero singular values, which is upper bounded by $\lesssim 1$. Using Von Neumann's trace inequalities, for any matrix  $\mA$, we have 
\begin{align}\label{eq:temp23}
 |\Tr \mA \mB^\top \mB | \leq \sum_{j}\sigma_j(\mA) \sigma_j(\mB^\top \mB)\lesssim N_< \|\mA\|_\op 
\end{align}
where $\sigma_j(\mA)$ is the $j$-th (in descending order) singular value of a matrix $\mA$. Now we proceed to control the second term.
Note that 
\begin{align}
  & (\mC - \mB^\top  \mB)^{-2} -  
    \mC ^{-2} 
    \\=& 
   (\mC - \mB^\top  \mB)^{-2} -  (\mC - \mB^\top  \mB)^{-1}\mC^{-1} + (\mC - \mB^\top  \mB)^{-1}\mC^{-1} -
    \mC ^{-2}
    \\
    =& (\mC - \mB^\top  \mB)^{-2} \mB^\top  \mB   \mC^{-1} 
    + (\mC - \mB^\top  \mB)^{-1} \mB^\top  \mB \mC^{-2}
\end{align}
As such, by \myeqref{eq:temp23} we have 
\begin{align}
    &|\Tr \mR(\mC - \mB^\top  \mB)^{-2} \mB^\top  \mB   \mC^{-1} \mR|/N_=
    \\=& 
    |\Tr \mC^{-1}  \mR^2(\mC - \mB^\top  \mB)^{-2} \mB^\top  \mB| /N_=
    \\\lesssim& N_</N_= \|\mC^{-1}  \mR^2(\mC - \mB^\top  \mB)^{-2}\|_\op\to 0 \,.
\end{align}
The other term can be bounded similarly. This finishes the proof of $\Delta_d'\to 0$ in probability. 
To sum up, we have the test error to be 
\begin{align}
    \mathrm{Err}(\mX;\lambda,  \vF,  \vh)  = 
    \Tr
    \left ( \mR^2\mC^{-2} \right) / N_=  \hat F_{r}^2  + 
    \left( 1 + m^{-1} \Tr \tilde \mD^{-2}   \overline \mZ_{\leq }^\top   \overline \mZ_{\leq }\right) \hat  F_{>r}^2
    + 
    \\m^{-1}\sigma_\epsilon^2  \Tr \tilde \mD^{-2}  \overline \mZ_{\leq }^\top   \overline \mZ_{\leq } 
    + 
    \Delta_d
\end{align}
\paragraph{Generalization Error via Marchenko-Pastur}
The next step is to reduce the traces to the integral form when $d\to\infty$. That is evaluating the followings as $d\to\infty$, 
\begin{align}
     &\Tr
    \left ( \mR^2\mC^{-2} \right) / N_= 
    \quad \mathrm{and }\quad 
    m^{-1}\sigma_\epsilon^2  \Tr \tilde \mD^{-2}  \overline \mZ_{\leq }^\top   \overline \mZ_{\leq }  
\end{align}
We begin with the simpler case $E_r=1$ and then consider $E_r>1$. 
\paragraph{
The $E_r=1$ case.} I.e., there is only one eigenspace with eigenvalues $\sim d^{-s_r}$. This is the case for one-hidden layer convolutional kernels and dot-product kernels. In this case, $n=0$ and 
\begin{align}
    \mR =  (\overline \xi_r^\upd)^{-1} \mI_{N_\leq}, \quad \mathrm{with} \quad 
    \overline\xi_r^\upd = \gamma^{-1} m(\sigma_{rn}^{(d)})^2  = \frac{m}{N_=} \hat h^2_r \gamma^{-1} \to \overline \xi_r = \alpha^{-1} \hat h^2_r \gamma^{-1} 
\end{align}
Choosing $\rho=(\overline \xi_r^\upd)^{-1}$ and applying Theorem \ref{Theorem:mp}, we have when\footnote{Note that $N_{\leq} = N_=(1 +o(1))$} $N_=/m\to\alpha\in (0, \infty)$ 
\begin{align}
    \overline \xi_r^{-2}
    \Tr
    \left ( \mC^{-2} \right) / N_= &\longrightarrow	 
    \int (1 + \overline \xi_r  t)^{-2} \mu_\alpha(t)d t\,\,
    \\
    \label{eq:finite-width-effect-variance}
    \frac{N_{\leq}}{m}  \frac 1 {N_{\leq}}  \Tr \tilde \mD^{-2}   \overline \mZ_{\leq }^\top   \overline \mZ_{\leq } 
    &\longrightarrow	 
    \alpha \overline \xi_r ^2
    \int {t}(1+ \overline \xi_r t)  ^{-2}\mu_\alpha(t) dt
\end{align}
Therefore, 
\begin{align}
    \mathrm{Err}(\mX;\lambda,  \vF,  \vh)  = \left(\hat F_r^2 \cdot 
    \int \frac{ \mu_\alpha(t)}{(1 + \overline \xi_r  t)^{2}}d t 
     + \hat F_{>r}^2  \right) 
    + \left( \hat F_{>r}^2 +\sigma_\epsilon^2\right) \,\cdot  
    \alpha \overline \xi_r ^2
    \int \frac{ t\mu_\alpha(t)}{(1 + \overline \xi_r  t)^{2}} dt + \Delta_d 
\end{align}
Both integrals have closed form formulas and they are computed in Sec.\ref{sec:computing integrals}. 

\paragraph{The $E_r>1$ Case.} Recall that $\mR$ is a diagonal matrix with entries 
$
\gamma (m(\sigma_{rn}^{(d)})^2)^{-1}   
$
whose multiplicity is $N_{rn}^{(d)}$. We assume 
the limiting density exist 
\begin{align}
\gamma (m(\sigma_{rn}^{(d)})^2)^{-1}  \to 
\gamma \alpha \hat h_{rn}^{-2}
\quad \mathrm{and} \quad 
    N_{rn}^{(d)} / \sum_{n\in [E_k]}  N_{rn}^{(d)} \to \tau_{rn} 
\end{align} 
and let $\nu_\vh(r)$ denote this distribution. For convenience, we still $\mR$ to represent a (sequence of) diagonal matrix with limiting spectral $\nu_\vh(r)$. By our assumptions on $\vh$, the support of $\nu_\vh(r)$ is bounded away from $0$ and $\infty$. Thus, ignoring vanishing correction term between $\overline \mZ_{\leq }^\top   \overline \mZ_{\leq } $ and $\overline \mZ_{= }^\top   \overline \mZ_{= } $, we need to compute the limit of the following 
\begin{align}
    &\frac 1 {N_\leq}\Tr \mR^2  (\mR + \overline \mZ_{\leq }^\top   \overline \mZ_{\leq } )^{-2} = 
    \mR^{1/2}  (1+ \mR^{-1/2}\overline \mZ_{\leq }^\top   \overline \mZ_{\leq } \mR^{-1/2})^{-1} \mR^{-1} (1+ \mR^{-1/2}\overline \mZ_{\leq }^\top   \overline \mZ_{\leq } \mR^{-1/2})^{-1}\mR^{1/2}\label{eqn:trace1}
    \\
    &\frac 1 {N_\leq}\Tr   (\mR + \overline \mZ_{\leq }^\top   \overline \mZ_{\leq } )^{-2}  \overline \mZ_{\leq }^\top   \overline \mZ_{\leq } = 
    \frac 1 {N_\leq}\Tr   \left((\mR + \overline \mZ_{\leq }^\top   \overline \mZ_{\leq } )^{-1}  -   (\mR + \overline \mZ_{\leq }^\top   \overline \mZ_{\leq } )^{-2}\mR  \right)\label{eqn:trace2} 
\end{align}
.

To evaluate the limit, we may need extra assumptions on the eigenfunctions $\phi_{knl}^\upd$ to ensure $\mR$ and $\mZ_{\leq }$ are asymptotically free. Nevertheless, under the freeness assumption, computing self-consistent equations that characterize the asymptotic values of the trace objects in~\myeqref{eqn:trace1} and~\myeqref{eqn:trace2} is then straightforward using tools from operator-valued free probability~\citep{mingo2017free}. We do not elaborate on the details here, but refer the reader so many related works for examples of how to apply these tools~\citep{far2006spectra,adlam2020neural,adlam2020understanding,tripuraneni2021covariate,tripuraneni2021overparameterization}.

\subsection{Computing the Integrals.}\label{sec:computing integrals}
It remains to compute the above integrals. Note that 
\begin{align}
    \overline \xi_r ^2
    \int \frac{ t\mu_\alpha(t)}{(1 + \overline \xi_r  t)^{2}} dt
    = \overline \xi_r \left( 
    \int \frac{ \mu_\alpha(t)}{(1 + \overline \xi_r  t)} dt
    -\int \frac{ \mu_\alpha(t)}{(1 + \overline \xi_r  t)^{2}} dt \right) 
\end{align}
As such we only need to compute, for $k=1$ and $k=2$, 
\begin{align}
    \zeta_\alpha(\xi; \alpha, k) = \int \frac{ \mu_\alpha(t)}{(1 +  \xi  t)^{k}} dt
\end{align}
Note that one only needs $\zeta_\alpha(\xi; \alpha, 1)$ as $\zeta_\alpha(\xi; \alpha, k)$ can be obtained from $\zeta_\alpha(\xi; \alpha, k-1)$ by taking derivative w.r.t. $\overline \xi$.
Denote $b_{\pm}=(1 \pm\sqrt{\alpha})^2$ and $\Delta = \alpha_+ - \alpha_-$. Then 
\begin{align}
    \mu_\alpha(t) = \left(1-\frac1\alpha\right)^{+} \delta_{0}(t) +  \frac{\sqrt{(\alpha_+ - t)(t-\alpha_-)}}{2\pi \alpha t}\1_{[\alpha_-, \alpha_+]}(t) dt
\end{align}
With $b=(1+\overline \xi_r \alpha_-)/(\overline \xi_r (\alpha_+ - \alpha_-))$ and $c = \overline \xi_r  \alpha_-/(\overline \xi_r (\alpha_+ - \alpha_-)) = \alpha_-/ (\alpha_+ - \alpha_-) $,  
\begin{align}
    &\int (1 + \overline \xi_r t)  ^{-k}\mu_\alpha (t) dt
    \\
    =& 
   \left(1-\frac1\alpha\right)^{+}  +
    \int_{[\alpha_-, \alpha_+]} (1 + \overline \xi_r t)  ^{-k}\frac{\sqrt{(\alpha_+ - t)(t-\alpha_-)}}{2\pi \alpha t} dt
    \\=&
    \left(1-\frac1\alpha\right)^{+} + 
    \frac{1}{2\pi\alpha\overline \xi_r} (\overline\xi_r (\alpha_+ - \alpha_-)) ^{1-k} \int_0^1 (b+t)^{-k} (c+t)^{-1} \sqrt{t(1-t)}dt 
\end{align}
Thanks to   
\href{https://www.wolframalpha.com/}{wolframalpha.com}, we have, after doing some algebra,  
\begin{align}
    \int_{0}^1 \frac{\sqrt{t(1-t)}}{(t+b)(t+c)} dt 
    % &= \pi \left( -1 +\frac{1+b+c}{\sqrt{bc(1+b)(1+c)}}\right)
    &= \pi \left( -1 +\frac{1+b+c}{\sqrt{b(1+b)} +\sqrt{c(1+c)}}\right)
    \\
    \int_{0}^1 \frac{\sqrt{t(1-t)}}{(t+b)^2(t+c)} dt &
    = \frac{\pi}{2 \sqrt{b^2+b}
    ((b+c+2bc ) +2 \sqrt{(b+1)(c+1)bc})
    }
\end{align}
\subsection{Proof of Lemma \ref{lemma:inv}.}\label{sec:proof-lemma-key}

Note that this lemma is trivial if $\lim_{d\to\infty}m/N_{\leq r}^\upd = \alpha^{-1}>1$ as 
$\overline \mZ_{\leq}^\top  \overline \mZ_{\leq}$ follows the Marchenko-Pastur distribution, and the smallest eigenvalues is bounded from below by $\alpha_{-} = (1-\sqrt{\alpha})^{2}$. When $\alpha^{-1}\leq 1$, we need to use the regularization term $\overline \mLambda_{\leq} ^{-1}$. We provide the details below.

Recall that $\overline \mZ_{\leq }^\top  = [\overline \mZ_<^\top , \overline \mZ_r^\top ] $, where $\overline \mZ_<^\top $ is a $m\times N_{<r}^\upd$ matrix consisting of low frequency modes and $\overline \mZ_r^\top $ is a $m\times N_r^{(d)}$ is a matrix consisting of critical frequency modes. We have $N_{<r}^\upd  \sim d^{s_{r-1}}$, $N_r^\upd \sim N_{\leq r}^\upd  \sim d^{s_r}$ and 
\begin{align}
    \overline \mZ_{< }^\top  \overline \mZ_{<} = \mI_{N_{<r}^\upd} + \Delta_d  
\end{align}
where $\E\|\Delta_d\|_{op} \to 0$ as $d\to\infty$ in probability. Let $\vu = [\beta_< \ve_<^\top , \beta_r \ve_r^\top ]^\top $ be a unit vector in $\sR^{N_{\leq r}^\upd}$, where $\ve_<$ and $\ve_r$ are unit vectors in $\sR^{N_{<r}^\upd}$ and $\sR^{N_{r}^\upd}$ resp., and $\beta_<^2+\beta_r^2=1$. We want to show that for some $\lambda_1>0$
\begin{align}
\lambda_1\leq     \vu^\top  D\vu 
= \vu^\top   \overline \mLambda_{\leq} ^{-1}  \vu 
+ \vu^\top  \overline \mZ_{\leq}^\top  \overline \mZ_{\leq} \vu 
\end{align}
Note that the entries in $\overline \mLambda_{\leq} ^{-1}$ corresponding to the critical-frequencies are $m(\sigma_{nr}^\upd)^2 \sim 1$. Thus there is a constant $c>0$ such that 
\begin{align}
    \vu^\top   \overline \mLambda_{\leq} ^{-1}  \vu \geq c \beta_r^2
\end{align}
In addition, if $C \equiv \|\overline \mZ_{r}\ve_r\|_2 $ then $C\leq  2\alpha_+$ in probability. 
Thus by the triangle inequality, 
\begin{align}
\vu^\top  D\vu 
&\geq 
c\beta_r^2 + \|\beta_<\overline \mZ_{<}\ve_< + \beta_r\overline \mZ_{r}\ve_r \|_2^2
\\&\geq 
c\beta_r^2 
+ (\|\beta_<\overline \mZ_{<}\ve_<\|_2 - \|\beta_r\overline \mZ_{r}\ve_r\|_2)^2
\\
&\geq 
c\beta_r^2  + ((1-\Delta_d) |\beta_<| - C|\beta_r|)^2
\end{align}
where $\Delta_d\to0$ in probability. 
If $C|\beta_r|< \frac 12  |\beta_<|$, the above is greater than $(1/2-\Delta_d) \beta_<^2 + c\beta_r^2 \gtrsim 1$; otherwise 
$C|\beta_r|\geq  \frac 12  |\beta_<|$ and 
$\vu^\top  D\vu \geq c\beta_r^2 \geq c (\frac 1{2C}\beta_<)^2$ and as a result 
\begin{align}
    \vu^\top  D\vu \geq  2 c\beta_r^2 /2 \geq    (c\beta_r^2 + c (\frac 1{2C}\beta_<)^2)/2 \geq  c \min(1, \frac1 {2C})^{2}/2 \gtrsim 1
\end{align}  

The other direction is easier as both $\overline \mLambda^{-1}_{\leq }$ and $\overline \mZ_{\leq}^\top  \overline \mZ_{\leq}$ have operator norms bounded above.  

\subsection{Proof of Claim \ref{claim-tail-estimate}} \label{sec:proof-of-claim-tail}
The proof is split into two part: the ultra-high frequency parts $k\geq j_0$ and the median-high-frequency part, $r<k<j_0$. The first part is done by a moment-based calculation and the second part is done by matrix concentration \citep{vershynin2010introduction}.
\paragraph{Controlling the Ultra-high-frequency.}
 Recall that 
\begin{align}
    \Delta_{kn}^\upd \equiv \frac 1 {N_{kn}^\upd} Z_{kn}(\mX) Z_{kn}(\mX)^\top  - \mI_m \,.
\end{align}
By {\bf Assumption (4.)}, the diagonals are zero and we have 
\begin{align}
\Delta_{kn}^\upd   &= [Z_{kn}(\vx_i)^\top Z_{kn}(\vx_j)/N_{kn}^\upd]_{i, j\in [m], i\neq j}  
\end{align}

Then 
\begin{align}
    \E \|\Delta_{kn}^\upd \|_\op^2 \leq \E \| \Delta_{kn}^\upd  \|_{\mathrm{F}}^2 
    =& \E\sum_{i\neq j} |Z_{kn}(\vx_i)^\top Z_{kn}(\vx_j)/N_{kn}^\upd|^2
    \\=&
     \frac 1 {(N_{kn}^\upd)^2}\E\sum_{i\neq j} \sum_{l, l'}\phi^\upd_{knl}(\vx_i)\phi^\upd_{knl}(\vx_j)\phi^\upd_{knl'}(\vx_i)\phi^\upd_{knl'}(\vx_j)
     \\=&
     \frac 1 {(N_{kn}^\upd)^2}\E\sum_{i\neq j} \sum_{l }\phi^\upd_{knl}(\vx_i)^2\phi^\upd_{knl}(\vx_j)^2
     \\=& \frac 1 {N_{kn}^\upd} m(m-1) \leq \frac 1 {N_{kn}^\upd} m^2
\end{align}
Recall that $E_k$ grows at most exponentially, i.e. $E_k\leq C^k$ for some constant $C$. Thus, choosing $d$ large enough such that $Cd^{-\delta_0/4}<1$ and summing over $k>j_0 \equiv [4s_r/\delta_0 + 4]+4 $, 
\begin{align}
     \E \sum_{k>j_0}\sum_{n\in E_k} \|\Delta_{kn}^\upd\|_\op 
    \lesssim& \sum_{k>j_0} C^k (m^2 / N_{kn}^\upd )^{1/2}
    \\ \lesssim& \sum_{k>j_0} C^k d^{-s_k/2 + s_r}
    \\ \leq & \sum_{k>j_0} C^k d^{-k \delta_0/ 2 + s_r} 
    \\\lesssim & \sum_{k>j_0} d^{-k \delta_0/ 4 + s_r} 
    \lesssim d^{-j_0\delta_0/4 +s_r}
    \lesssim d^{-\delta_0}
\end{align}

\paragraph{Controlling the Median-high-frequency.}
It remains to show, for some $\epsilon>0$ 
\begin{align}
     \E \sum_{r< k\leq  j_0}\sum_{n\in E_k}\|\Delta_{kn}^\upd\|_\op  \lesssim d^{-\epsilon} \,.
\end{align}
As there are only finitely many terms in this sum, we only need to show that for each $k$ and $n$, 
$$
\E\|\Delta_{kn}^\upd\|_\op \lesssim d^{-\epsilon}\,. 
$$
We use the following theorem from \citet{vershynin2010introduction} regarding matrix concentration to prove this claim. 

\begin{theorem}[Theorem 5.62 \citet{vershynin2010introduction}] Let $\mA$ be a $N\times m$ ($N\geq m$) matrix whose columns $A_j$ are independent isotropic random vectors in $\sR^N$ with $\|A_j\|_2=N$ a.s. Let $K$ be defined as 
\begin{align}
    K = \frac 1 N \E \max_{j\leq m} \sum_{i\in [m], i \neq j} |A_j^\top A_i|^2
\end{align}
Then 
\begin{align}
    \E \| \mA^\top  \mA /N -\mI_m\|_\op \lesssim \sqrt{K \log(m)/N} 
\end{align}
\end{theorem}
We apply this theorem to $\mA = Z_{kn}(\mX)^\top $. The columns of $Z_{kn}(\mX)^\top $ are $A_j=\phi^\upd_{kn}(\vx_j)$, $j\in [m]$ which are independent. Let $N = N_{kn}^\upd$. By Assumption (4.),
\begin{align}
    A_j^\top A_j = \sum_{l}\phi^\upd_{knl}(\vx_j)^2 = N. 
\end{align}
We claim that $K\lesssim_{k,q} m^{1+ \frac 1 q}$ for any $q\geq 1$.  Indeed, let 
\begin{align}
    B_{j} = \sum_{i\in [m], i \neq j} |A_j^\top A_i|^2
\end{align}
We then remove the maximal function by paying an $m^{1/q}$ factor  
\begin{align}
    K = \frac 1 N \E \max_{j\leq m} B_j 
    \leq \frac 1 N m^{1/q }|\E B_j^q  |^{1/q}
     = \frac 1 N m^{1/q }\left|\E (\sum_{i\in [m], i \neq j} |A_j^\top A_i|^2)^{2q /2}  \right|^{1/q}
\end{align}
Next we apply the Minkowski inequality to swap the $L^{2q}$-norm and the $l^2$-norm, 
\begin{align}
    K \leq \frac 1 N m^{1/q }\sum_{i\in [m], i \neq j}  (\E |A_j^\top A_i|^{2q})^{1/2q  \times 2 }
    \leq \frac 1 N m^{1/q +1 } (\E |A_j^\top A_i|^{2q})^{1/2q  \times 2 } \leq C_{k,q}^2 m^{1/q + 1}
\end{align}
if $(\E |A_j^\top A_i|^{2q})^{1/2q  \times 2 } \leq C_{k, q}^2 N$, which is done by hypercontractivities below. Indeed, for $\vx_j$ fixed, $Z_{rn}(\vx_j)^\top Z_{rn}(\vx_i)$ is a linear combination of $\phi_{knl}^\upd$, {\bf Assumption (2.)} gives 
\begin{align}
    \E_{\vx_i} |A_j^\top A_i|^{2q}
    =& \E_{\vx_i} |Z_{rn}(\vx_j)^\top Z_{rn}(\vx_i)|^{2q}
    \\\leq&  \left( C_{k, q} (\E_{\vx_i} |Z_{rn}(\vx_j)^\top Z_{rn}(\vx_i)|^{2})^{1/2}\right)^{2q}
    \\=& C_{k, q}^{2q} N^q 
\end{align}
where we applied  
\begin{align}
    \E_{\vx_i} |Z_{rn}(\vx_j)^\top Z_{rn}(\vx_i)|^{2}
    =& \E_{\vx_i} |\sum_{l} \phi_{knl}^\upd(\vx_i)\phi_{knl}^\upd(\vx_j)|^2
    \\=& \E_{\vx_i} \sum_{ll'} \phi_{knl}^\upd(\vx_i)\phi_{knl}^\upd(\vx_j)
    \phi_{knl'}^\upd(\vx_i)\phi_{knl'}^\upd(\vx_j)
     \\=& \sum_{l}\phi_{knl}^\upd(\vx_j)^2 
      \\=& N 
\end{align}
Therefore with $\mA  = Z_{kn}(\mX)^\top $, we have 
\begin{align}
\E \|\Delta_{kn}^\upd\|_\op =    \E \|Z_{kn}(\mX) Z_{kn}(\mX)^\top  /N_{kn}^\upd -\mI_m\|_\op
    \lesssim_{k, q} \sqrt{m^{1+ 1/q }\log m /N_{kn}^\upd} \,.
\end{align}
For each fixed $k > r$, $s_k -s_r\geq \delta_0$, by choosing $q$ sufficiently large (depending on $k$ and $\delta_0$), we have 
\begin{align}
    \E \|\Delta_{kn}^\upd\|_\op \lesssim_{k} d^{-\delta_0/2}\,.
\end{align}

\section{Additional Plots}\label{sec:additional-plots}
To simulate the learning curves for higher-order scalings, e.g. $r=4$, we must chose $d$ small. As such, we are in a strong finite-size correction regime. In this section, we vary $d$ to visualize the finite-size effect of the predictions. Note that for larger $d$ ($=60$ here), we can only simulate up to the quadratic scaling. 
For smaller $d$ ($d=10$), we observe noticeable finite-size correction. However, the predictions match the simulations quite well. When $d$ become larger $d=60$, the predicted learning curves match the simulation almost perfectly.   
\begin{figure}
    \centering
    \includegraphics[width=.59\textwidth]{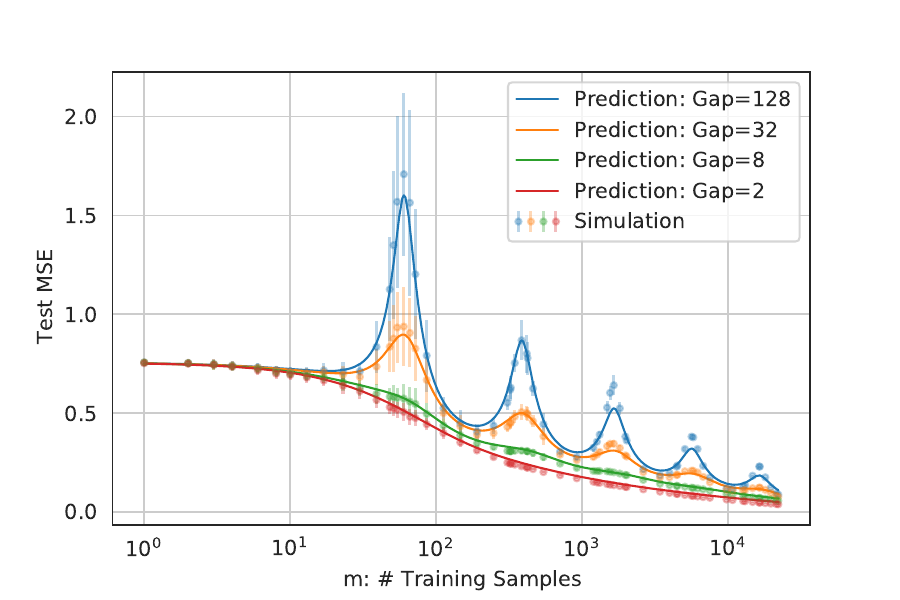}
         \caption{{\bf Tiny $d=10$ }}
    \label{fig:tiny-d-multiple-descent}
    \includegraphics[width=.59\textwidth]{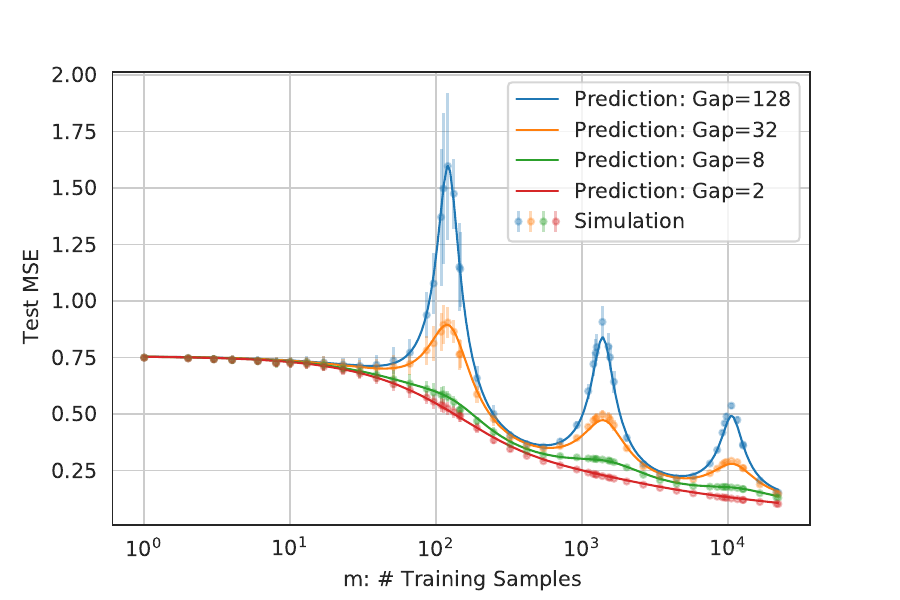}
        \caption{{\bf Small $d=20$}}
    \label{fig:small-d-multiple-descent}
    \includegraphics[width=.59\textwidth]{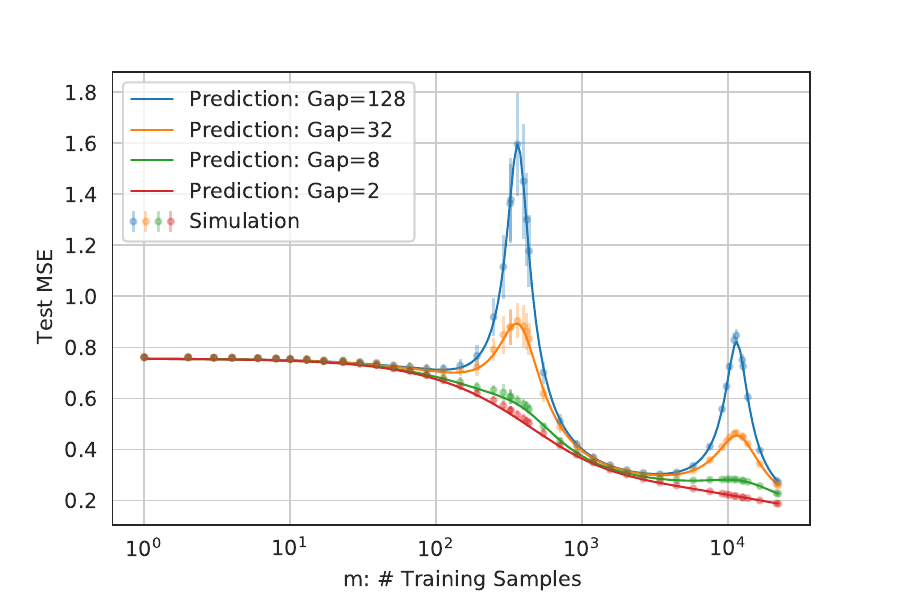}
    \caption{{\bf Large $d=60$ }}
    \label{fig:large-d-multiple-descent}
\end{figure}

\section{Further Analysis}

\subsection{Reducing finite-size Effect.}
There are two non-obvious improvements in our results that lead to near perfect agreements between simulations and predictions even for small $d$. The first one is to use $m=N(d, \leq r)$ to compute $r$-th peak vs. $m=N(d, r)$ (or $m=d^r/r!$). As it is shown in Fig.~\ref{fig:correcting-finite-size-effect} (a), using $m=N(d, r)$ as the peak in the theoretical prediction, the prediction is a bit off to the left. The second one is to use the sum over all contributions from all critical scaling $m=N(d, \leq r)$ (i.e., \myeqref{eq:learning-curves} rather than the contribution from a single critical scaling (i.e., \myeqref{eq:generalization formula}.) As it is shown in Fig.~\ref{fig:correcting-finite-size-effect} (b), the predictions are a bit smaller than the simulations when using the latter. These two improvements together lead to accurate agreement Fig.~\ref{fig:correcting-finite-size-effect} (c).

\begin{figure}
     \centering
     \begin{subfigure}[b]{0.49\textwidth}
         \centering
         \includegraphics[width=\textwidth]{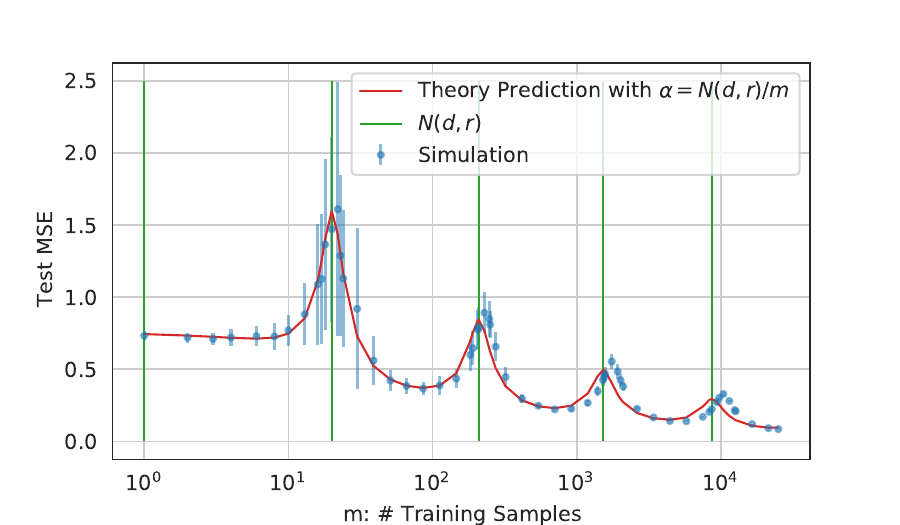}
         \caption{Using $\alpha=N(d,r)/m$}
         \label{fig:y equals x}
     \end{subfigure}
     \hfill
     \begin{subfigure}[b]{0.49\textwidth}
         \centering
         \includegraphics[width=\textwidth]{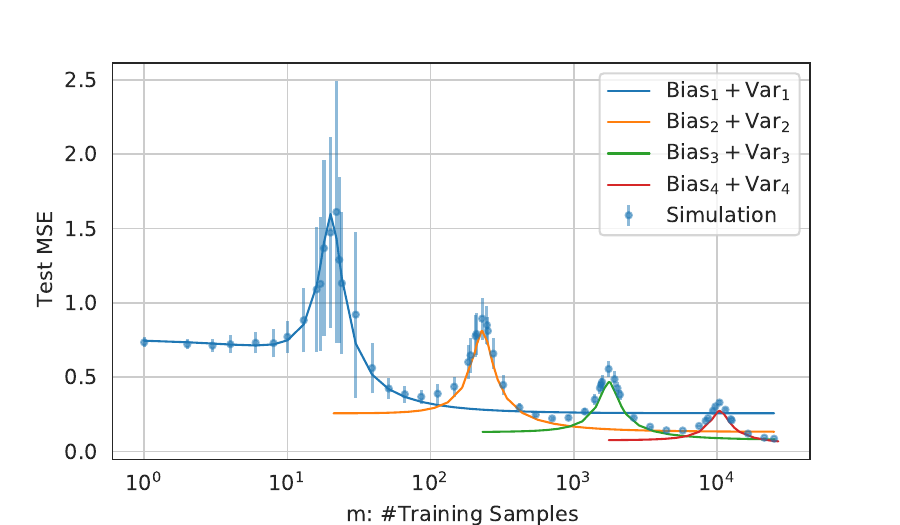}
         \caption{Using \myeqref{eq:generalization formula}}
     \end{subfigure}
     \hfill
     \begin{subfigure}[b]{0.49\textwidth}
         \centering
         \includegraphics[width=\textwidth]{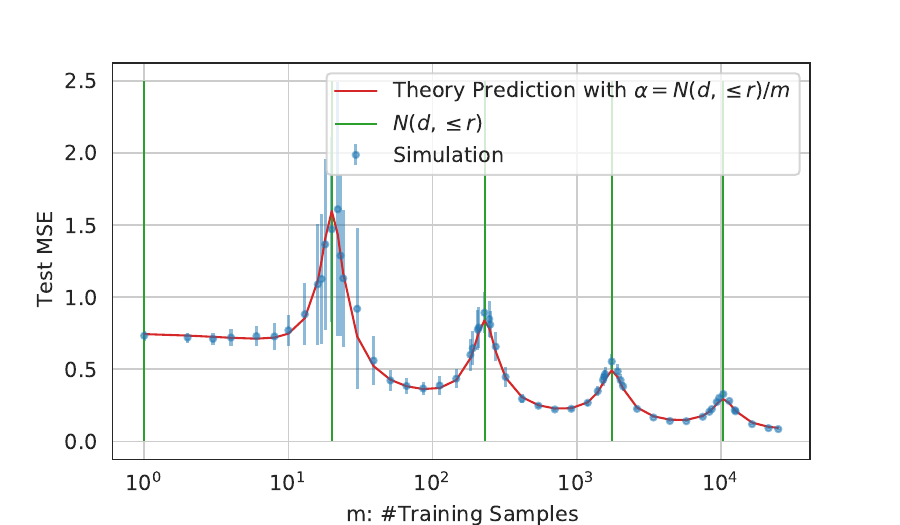}
         \caption{Using \myeqref{eq:learning-curves}}
         \label{fig:five over x}
     \end{subfigure}
        \caption{{\bf Two improvements reduce the finite-size effect.} (a) The theoretical prediction is a bit off to the left when estimating  $\alpha$ using $N(d,r)/m$. (b) The prediction from \myeqref{eq:generalization formula} is a bit smaller than the simulation due to the finite-size effect. 
        (c) Almost perfect agreement between the prediction and the simulation after two improvements (1) replacing \myeqref{eq:generalization formula} by \myeqref{eq:learning-curves} and (2) estimating $\alpha$ with $N(d, \leq r)/m$ rather than  $N(d, r)/m$. Here $d=20$ and $p=1$, i.e., inner product kernel. 
        }
        \label{fig:correcting-finite-size-effect}
\end{figure}

\subsection{Choosing the number of peaks by choosing the right regularization.}
Recall that the height of the $r$-th variance term scales like  
\begin{align}
        \xi_r(\hat\vh,\lambda, 1)^{1/2} = \left(\frac{\hat h_r^2}{\lambda + \hat h_{>r}^2}\right)^{1/2}. 
\end{align}
If $\hat h_r^2 \gg \hat h_{>r}^2$ and $\lambda  \leq \hat h_{>r}^2$, then $\xi_r(\hat\vh, \lambda)^{1/2}$ is large, which could lead to a peak at $m=N(d,\leq r)$. To eliminate this peak, we could choose $\lambda \sim \hat h_r^2$ which implies $\xi_r(\hat\vh,\lambda, 1)\lesssim 1$. We verify this observation in Fig.~\ref{fig:controlling-peaks}. When $\lambda=0$, the unregularized learning curve have 4 peaks. By increasing $\lambda$ to $1e-7, 1e-5, 1e-3, 1e-1$, the number of peaks are reduced to 3, 2, 1, 0, respectively. A similar result has also been observed in a linear design setting \citep{wu2020optimal}. The similarity between the linear design in \citep{wu2020optimal} and the nonlinear design here shouldn't be surprising, as we prove a "Gaussian equivalence conjecture," which implies that the polynomial scalings are essential "replicas" of linear designs with different scales.

\begin{figure}
    \centering
    \begin{subfigure}[b]{0.49\textwidth}
    \includegraphics[width=\textwidth]{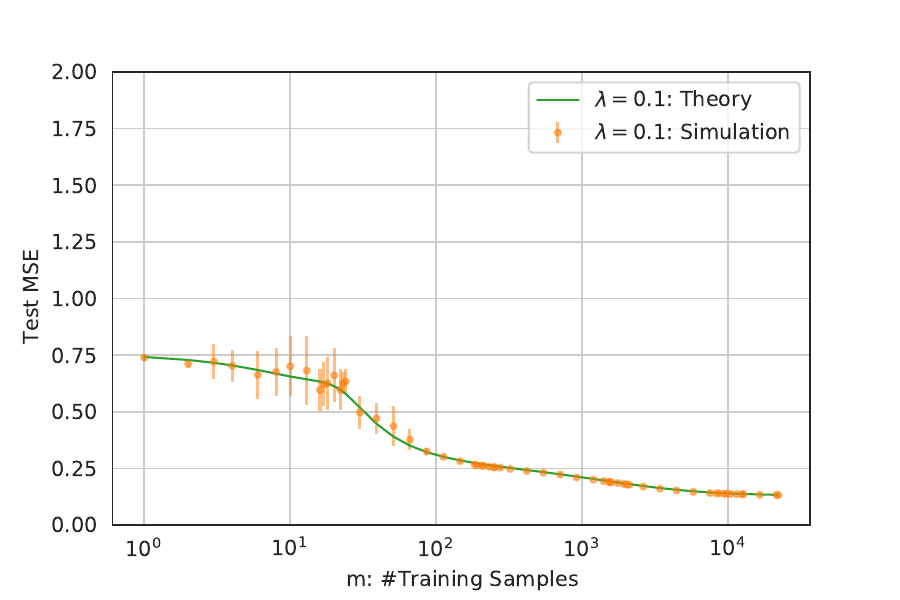}
     \caption{No peak: $\lambda=0.1$}
    \end{subfigure}
    \begin{subfigure}[b]{0.49\textwidth}
    \includegraphics[width=\textwidth]{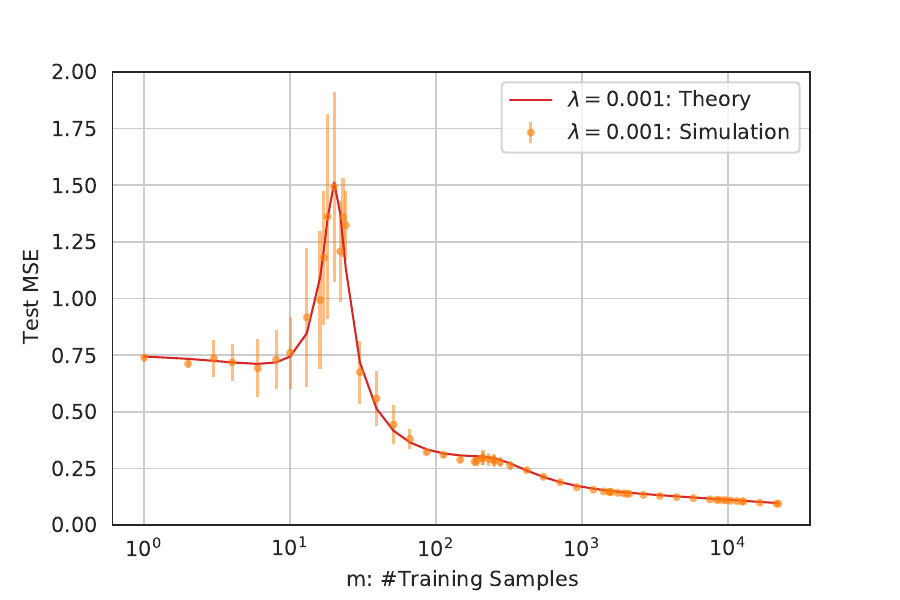}
         \caption{One peak: $\lambda=10^{-3}$}
    \end{subfigure}
    \begin{subfigure}[b]{0.49\textwidth}
    \includegraphics[width=\textwidth]{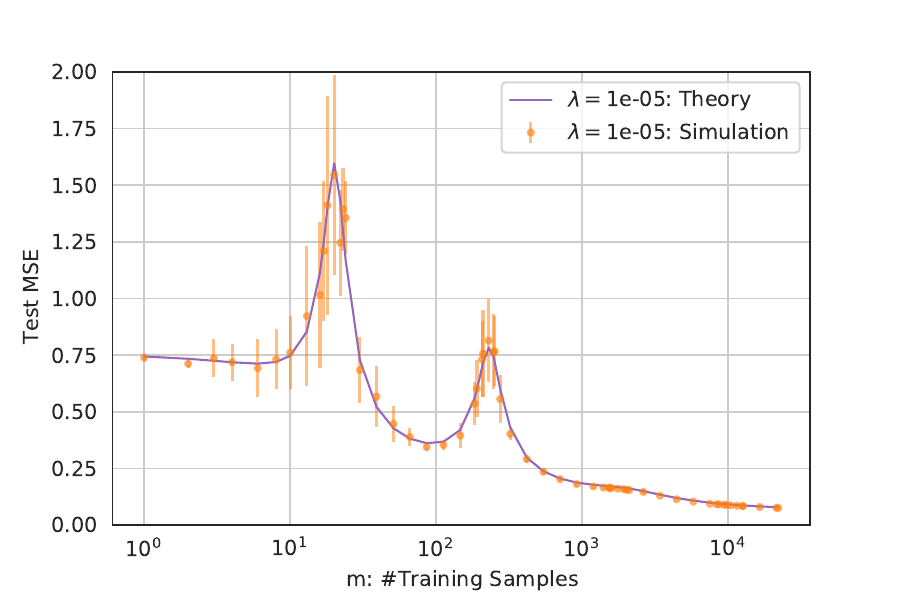}
             \caption{Two peaks: $\lambda=10^{-5}$}
    \end{subfigure}
    \begin{subfigure}[b]{0.49\textwidth}
    \includegraphics[width=\textwidth]{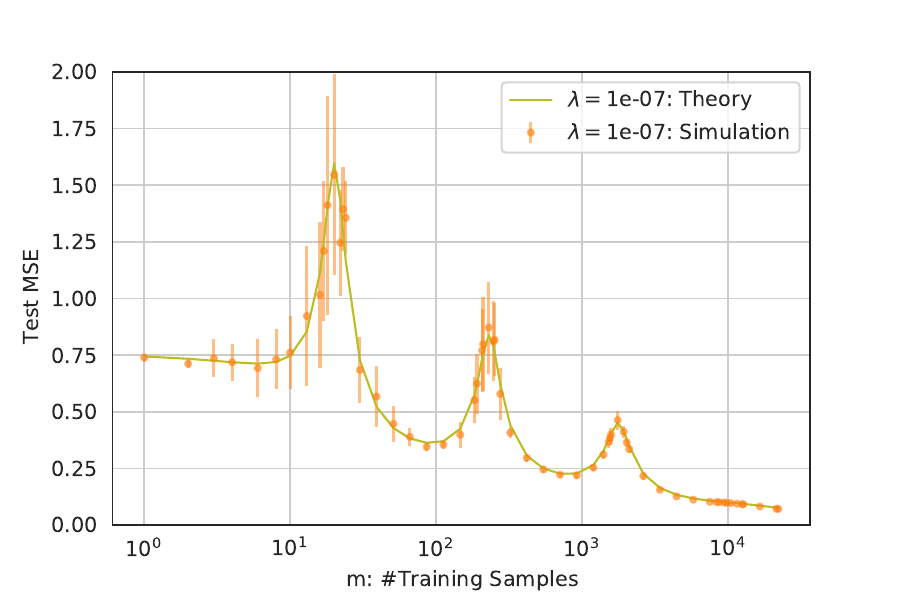}
                 \caption{Three peaks: $\lambda=10^{-7}$}
    \end{subfigure}
    \begin{subfigure}[b]{0.49\textwidth}
    \includegraphics[width=\textwidth]{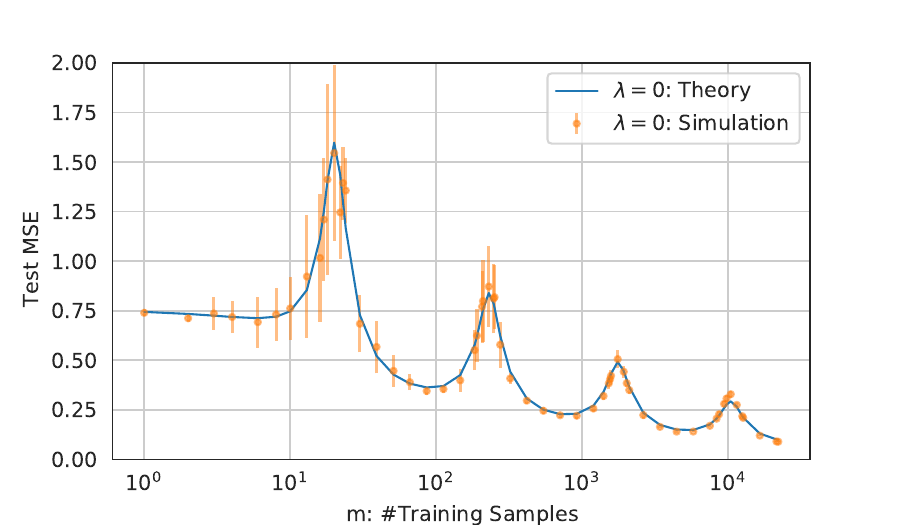}
                 \caption{Four peaks: $\lambda=0$}
    \end{subfigure}
        \begin{subfigure}[b]{0.49\textwidth}
    \includegraphics[width=\textwidth]{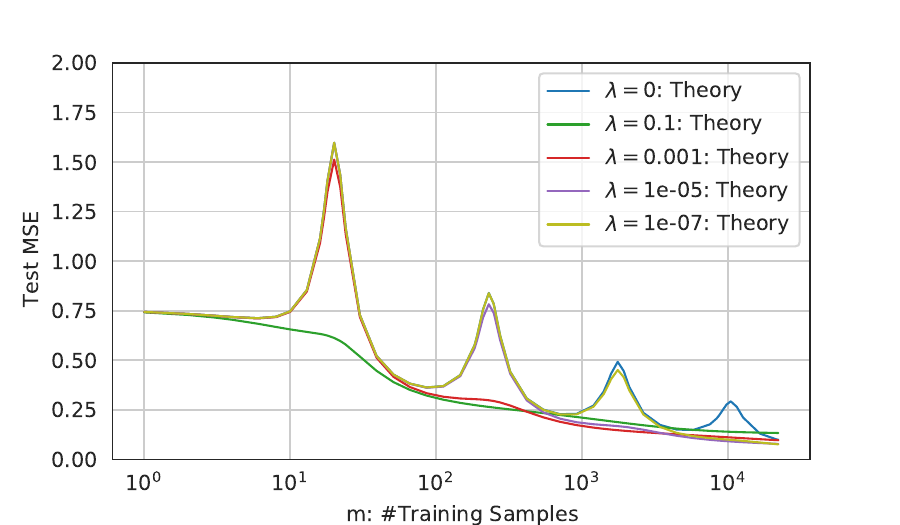}
                 \caption{\# Peaks vs Regularization.}
    \end{subfigure}

    \caption{{\bf Controlling the number of peaks by varying the strength of regularization.} }
    \label{fig:controlling-peaks}
\end{figure}

\subsection{Natural Data vs. Spherical Data}
We compare the spectrum of the NTKs of CIFAR10 associated with three architectures (FCN: fully-connected networks, CNN-VEC: convolutional networks without pooling, and CNN-GAP: convolutional networks with a global average pooling) against the one-layer convolutional kernels with spherical-type of data. Recall that the larger spectral gap between eigenspaces triggers the multiple-descent phenomenon. This phenomenon disappears, and the learning curve becomes monotonic when the spectral gap is small. Fortunately, for CIFAR10, the spectrum of the NTKs are continuous, and the learning curves are monotonic (power-law decay.) As such, there is still a gap between our results/assumptions and natural data.

\begin{figure}
    \centering
    \begin{subfigure}[b]{0.49\textwidth}
    \includegraphics[width=\textwidth]{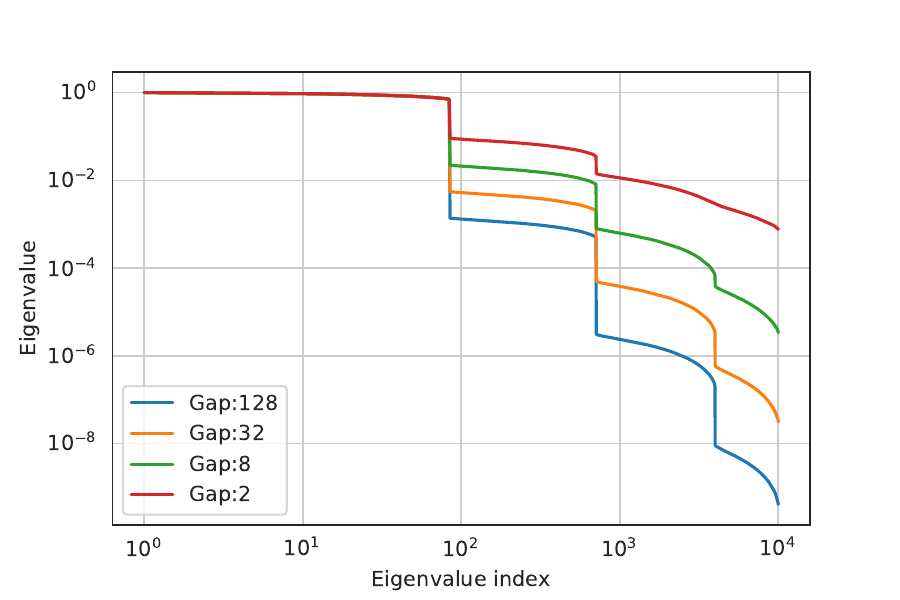}
    \caption{Spectrum: Spherical Data}
    \end{subfigure}
    \begin{subfigure}[b]{0.49\textwidth}
    \includegraphics[width=\textwidth]{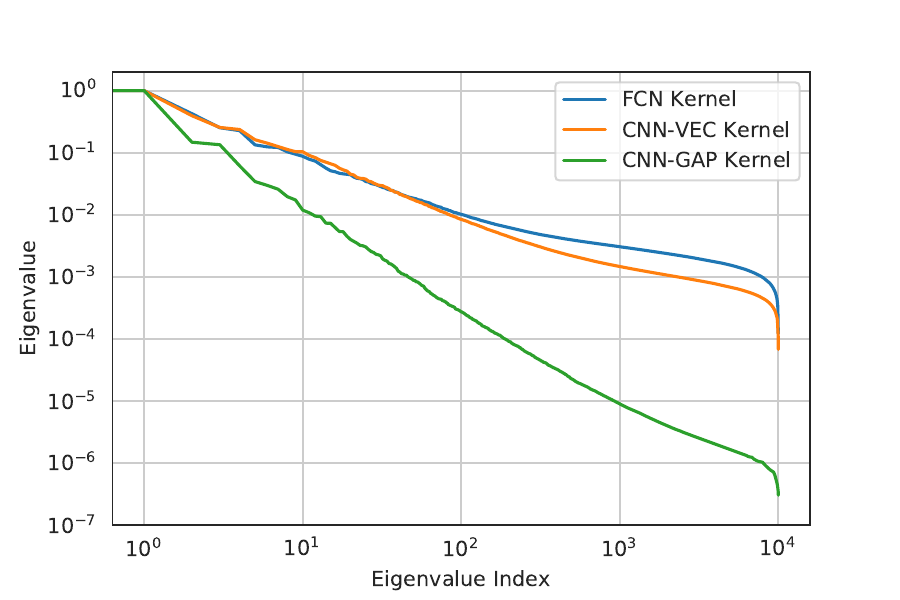}
    \caption{Spectrum: CIFAR10}
    \end{subfigure}
    \begin{subfigure}[b]{0.49\textwidth}
    \includegraphics[width=\textwidth]{multiple-descent-var-sp-theory-sim-d20-p=6.pdf}
    \caption{Test MSE: Spherical Data}
    \end{subfigure}
    \begin{subfigure}[b]{0.49\textwidth}
    \includegraphics[width=\textwidth]{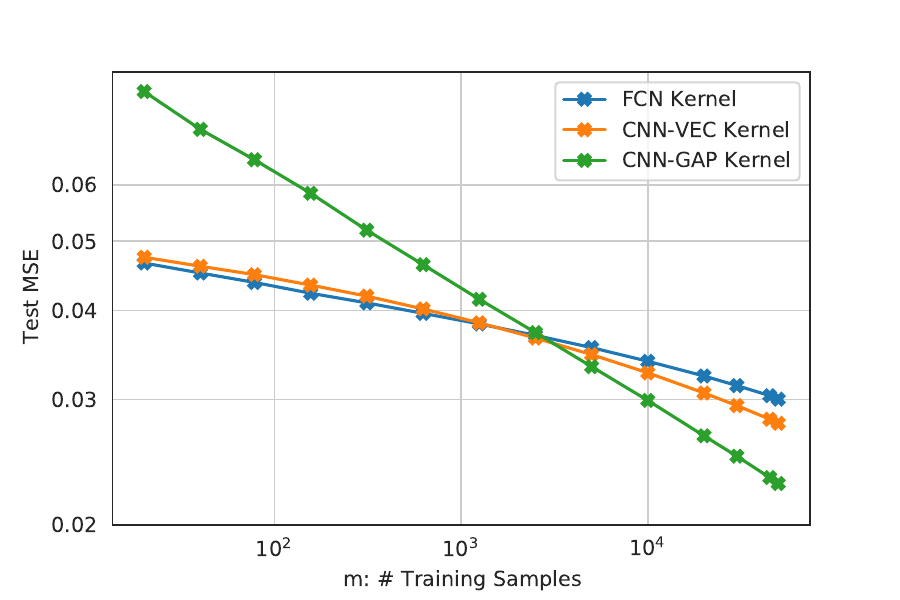}
        \caption{Test MSE: CIFAR10}
    \end{subfigure}
\caption{ {\bf Spectrum (top) and learning curves (bottom) of Spherical data (left) vs. CIFAR10 (right.)} The spectrum of CIFAR10 has a power-law decay and does not contain any sizable spectral gap, which is the main cause of the multiple-descent phenomena. The learning curves of CIFAR10 have power-law decay for all three kernels.}
    \label{fig:my_label}
\end{figure}

\end{document}